\documentclass[11pt]{article}
\pdfoutput=1
\usepackage{enumerate}
\usepackage{pdfsync}
\usepackage{appendix}
\usepackage[OT1]{fontenc}
\usepackage{smile}
\usepackage[colorlinks,
            linkcolor=red,
            anchorcolor=blue,
            citecolor=blue
            ]{hyperref}
\usepackage{color, colortbl}
\usepackage{fullpage}
\usepackage[protrusion=true,expansion=true]{microtype}
\usepackage{pbox}
\usepackage{setspace}
\usepackage{tabularx}
\usepackage{float}
\usepackage{wrapfig,lipsum}
\usepackage{enumitem}
\usepackage{microtype}
\usepackage{graphicx}
\usepackage{subfigure}
\usepackage{enumerate}
\usepackage{enumitem}
\usepackage{pgfplots}
\usetikzlibrary{arrows,shapes,snakes,automata,backgrounds,petri}
\usepackage{booktabs} 

\usepackage{smile}

\usepackage{hyperref}
\usepackage{url}
\usepackage{fvextra}
\usepackage{mdframed}
\usepackage{xspace}
\usepackage{subcaption}
\usepackage{wrapfig}
\usepackage{placeins}
\usepackage{tabularx}
\usepackage{fancyvrb}
\usepackage[utf8]{inputenc} 
\usepackage[T1]{fontenc}    
\usepackage{hyperref}       
\usepackage{url}            
\usepackage{booktabs}       
\usepackage{amsfonts}       
\usepackage{amsmath}
\usepackage{amssymb}
\usepackage{algorithm}
\usepackage{algorithmic}
\usepackage{nicefrac}       
\usepackage{microtype}      
\usepackage{xcolor}         
\usepackage{enumitem}

\usepackage{tabularx}
\usepackage{array}
\usepackage{bm}
\usepackage{adjustbox}
\usepackage{graphicx}

\usepackage{tikz}
\usetikzlibrary{arrows.meta}

\let\hat\widehat
\let\tilde\widetilde
\def\given{{\,|\,}}

\def\biggiven{\,\big{|}\,}

\newcommand{\BB}{\mathbb{B}}
\newcommand{\hB}{\hat{\mathbb{B}}}

\ifx\BlackBox\undefined
\newcommand{\BlackBox}{\rule{1.5ex}{1.5ex}}  
\fi

\ifx\QED\undefined
\def\QED{~\rule[-1pt]{5pt}{5pt}\par\medskip}
\fi

\ifx\proof\undefined
\newenvironment{proof}{\par\noindent{\bf Proof\ }}{\hfill\BlackBox\\[2mm]}
\fi

\ifx\theorem\undefined
\newtheorem{theorem}{Theorem}
\fi
\ifx\example\undefined
\newtheorem{example}[theorem]{Example}
\fi
\ifx\property\undefined
\newtheorem{property}{Property}
\fi
\ifx\lemma\undefined
\newtheorem{lemma}[theorem]{Lemma}
\fi
\ifx\proposition\undefined
\newtheorem{proposition}[theorem]{Proposition}
\fi
\ifx\remark\undefined
\newtheorem{remark}[theorem]{Remark}
\fi
\ifx\corollary\undefined
\newtheorem{corollary}[theorem]{Corollary}
\fi
\ifx\definition\undefined
\newtheorem{definition}[theorem]{Definition}
\fi
\ifx\conjecture\undefined
\newtheorem{conjecture}[theorem]{Conjecture}
\fi
\ifx\fact\undefined
\newtheorem{fact}[theorem]{Fact}
\fi
\ifx\claim\undefined
\newtheorem{claim}[theorem]{Claim}
\fi
\ifx\assumption\undefined
\newtheorem{assumption}[theorem]{Assumption}
\fi

\newcommand{\la}{\langle}
\newcommand{\ra}{\rangle}

\newcommand{\hatpistar}{\hat{\pi}^*}

\newcommand{\actions}{\mathcal{A}}

\newcommand{\E}[2]{\mathbb{E}_{#1} \left[#2\right]}

\makeatletter
\newcommand*{\rom}[1]{\expandafter\@slowromancap\romannumeral #1@}
\makeatother

\allowdisplaybreaks
\title{Provable Zero-Shot Generalization in Offline Reinforcement Learning}

\author{%
  Zhiyong Wang \thanks{The Chinese University of Hong Kong; e-mail: {\tt zhiyongwangwzy@gmail.com}}~~
  Chen Yang \thanks{Indiana University Bloomington; e-mail: {\tt cya2@iu.edu}}~~
  John C.S. Lui \thanks{The Chinese University of Hong Kong; e-mail: {\tt cslui@cse.cuhk.edu.hk}}~~
	Dongruo Zhou \thanks{Indiana University Bloomington; e-mail: {\tt dz13@iu.edu}} 
}

\begin{document}

\date{}
\maketitle

\begin{abstract}
In this work, we study offline reinforcement learning (RL) with zero-shot generalization property (ZSG), where the agent has access to an offline dataset including experiences from different environments, and the goal of the agent is to train a policy over the training environments which performs well on test environments without further interaction. Existing work showed that classical offline RL fails to generalize to new, unseen environments. We propose pessimistic empirical risk minimization (PERM) and pessimistic proximal policy optimization (PPPO), which leverage pessimistic policy evaluation to guide policy learning and enhance generalization. We show that both PERM and PPPO are capable of finding a near-optimal policy with ZSG. Our result serves as a first step in understanding the foundation of the generalization phenomenon in offline reinforcement learning.
\end{abstract}

\section{Introduction}

Offline reinforcement learning (RL) has become increasingly significant in modern RL because it eliminates the need for direct interaction between the agent and the environment; instead, it relies solely on learning from an offline training dataset. However, in practical applications, the offline training dataset often originates from a different environment than the one of interest. This discrepancy necessitates evaluating RL agents in a generalization setting, where the training involves a finite number of environments drawn from a specific distribution, and the testing is conducted on a distinct set of environments from the same or different distribution. This scenario is commonly referred to as the zero-shot generalization (ZSG) challenge which has been studied in online RL\citep{Rajeswaran2017TowardsGA, Machado2018RevisitingTA, Justesen2018IlluminatingGI, Packer2018AssessingGI, Zhang2018ADO, Zhang2018ASO}, as the agent receives no training data from the environments it is tested on.

A number of recent empirical studies \citep{mediratta2023generalization, yang2023essential,mazoure2022improving} have recognized this challenge and introduced various offline RL methodologies that are capable of ZSG. Notwithstanding the lack of theoretical backing, these methods are somewhat restrictive; for instance, some are only effective for environments that vary solely in observations\citep{mazoure2022improving}, while others are confined to the realm of imitation learning\citep{yang2023essential}, thus limiting their applicability to a comprehensive framework of offline RL with ZSG capabilities. Concurrently, theoretical advancements \citep{bose2024offline,ishfaq2024offline} in this domain have explored multi-task offline RL by focusing on representation learning. These approaches endeavor to derive a low-rank representation of states and actions, which inherently requires additional interactions with the downstream tasks to effectively formulate policies based on these representations. Therefore, we raise a natural question:
\begin{center}
    \emph{Can we design provable offline RL with zero-shot generalization ability?}
\end{center}


We propose novel offline RL frameworks that achieve ZSG to address this question affirmatively. Our contributions are listed as follows.
\begin{itemize}[leftmargin = *]
   \item We first analyze when existing offline RL approaches fail to generalize without further algorithm modifications. Specifically, we prove that if the offline dataset does not contain context information, then it is impossible for vanilla RL that equips a Markovian policy to achieve a ZSG property. We show that the offline dataset from a contextual Markov Decision Process (MDP) is not distinguishable from a vanilla MDP which is the average of contextual Markov Decision Process over all contexts. Such an analysis verifies the necessity of new RL methods with ZSG property. 
    \item We propose two meta-algorithms called pessimistic empirical risk minimization (PERM) and pessimistic proximal policy optimization (PPPO) that enable ZSG for offline RL \citep{jin2021pessimism}. In detail, both of our algorithms take a pessimistic policy evaluation (PPE) oracle as its component and output policies based on offline datasets from multiple environments. Our result shows that the sub-optimalities of the output policies are bounded by both the supervised learning error, which is controlled by the number of different environments, and the reinforcement learning error, which is controlled by the coverage of the offline dataset to the optimal policy. Please refer to Table \ref{tab:example}
for a summary of our results. To the best of our knowledge, our proposed algorithms are the first offline RL methods that provably enjoy the ZSG property.

\end{itemize}

\begin{table*}[t!]
\centering
\caption{Summary of our algorithms and their suboptimality gaps, where $\cA$ is the action space, $H$ is the length of episode, $n$ is the number of environments in the offline dataset. Note that in the multi-environment setting, $\pi^*$ is the near-optimal policy w.r.t. expectation (defined in Section \ref{sec:setting}). $\mathcal{N}$ is the covering number of the policy space $\Pi$ w.r.t. distance $\mathrm d(\pi^1,\pi^2) = \max_{s\in \mathcal{S}, h \in [H]} \|\pi^1_h(\cdot|s) - \pi^2_h (\cdot|s)\|_{1}$. The uncertainty quantifier $\Gamma_{i,h}$ are tailored with the oracle return in the corresponding algorithms (details are in Section \ref{sec:withcontext}). }
\label{tab:example}

\begin{tabular}{|c|c|}
\hline
\textbf{Algorithm} & \textbf{Suboptimality Gap}\\
\hline
\makecell{PERM (our Algo.\ref{alg:erm})} & 
\makecell{
  $\sqrt{\log(\mathcal{N})/n} + n^{-1}\sum_{i=1}^n \sum_{h=1}^H$
  $\EE_{i,\pi^*}\big[\Gamma_{i,h}(s_h,a_h)\,\big\vert\, s_1=x_1\big]$
}
\\
\hline
\makecell{PPPO (our Algo.\ref{alg:modelfree})} &
\makecell{
  $\sqrt{\log|\actions|\,H^2/n} + n^{-1}\sum_{i=1}^n \sum_{h=1}^H$
  $\EE_{i,\pi^*}\big[\Gamma_{i,h}(s_h,a_h)\,\big\vert\, s_1=x_1\big]$
}
\\
\hline
\end{tabular}
\end{table*}

\noindent\textbf{Notation} 
We use lower case letters to denote scalars, and use lower and upper case bold face letters to denote vectors and matrices respectively. We denote by $[n]$ the set $\{1,\dots, n\}$. For a vector $\xb\in \RR^d$ and a positive semi-definite matrix $\bSigma\in \RR^{d\times d}$, we denote by $\|\xb\|_2$ the vector's Euclidean norm and define $\|\xb\|_{\bSigma}=\sqrt{\xb^\top\bSigma\xb}$. For two positive sequences $\{a_n\}$ and $\{b_n\}$ with $n=1,2,\dots$, 
    we write $a_n=O(b_n)$ if there exists an absolute constant $C>0$ such that $a_n\leq Cb_n$ holds for all $n\ge 1$ and write $a_n=\Omega(b_n)$ if there exists an absolute constant $C>0$ such that $a_n\geq Cb_n$ holds for all $n\ge 1$. We use $\tilde O(\cdot)$ to further hide the polylogarithmic factors. 
We use $(x_i)_{i=1}^n$ to denote sequence $(x_1, ..., x_n)$, and we use $\{x_i\}_{i=1}^n$ to denote the set $\{x_1, ...,x_n\}$. We use $\text{KL}(p\|q)$ to denote the KL distance between distributions $p$ and $q$, defined as $\int p\log(p/q)$. We use $\EE[x],\mathbb{V}[x] $ to denote expectation and variance of a random variable $x$.

The remaining parts are organized as follows. In Section \ref{related}, we discuss related works. In Section \ref{sec:setting}, we introduce the setting of our work. In Section \ref{sec:nocontext}, we analyze when existing offline RL approaches \citep{jin2021pessimism} fail to generalize without further algorithm modifications. In Section \ref{sec:withcontext}, we introduce our proposed meta-algorithms and provide their theoretical guarantees. In Section \ref{sec:linear}, we specify our meta-algorithms and analysis to a more concrete linear MDP setting. Finally, in Section \ref{sec: conclusion}, we conclude our work and propose some future directions.

\section{Related works}\label{related}

\noindent\textbf{Offline RL}
Offline reinforcement learning (RL) \citep{ernst2005tree,riedmiller2005neural,lange2012batch,levine2020offline} addresses the challenge of learning a policy from a pre-collected dataset without direct online interactions with the environment. A central issue in offline RL is the inadequate dataset coverage, stemming from a lack of exploration \citep{levine2020offline,liu2020provably}. A common strategy to address this issue is the application of the pessimism principle, which penalizes the estimated value of under-covered state-action pairs. Numerous studies have integrated pessimism into various single-environment offline RL methodologies. This includes model-based approaches \citep{rashidinejad2021bridging,uehara2021pessimistic,jin2021pessimism,yu2020mopo,xie2021policy,uehara2021representation,yin2022near}, model-free techniques \citep{kumar2020conservative,wu2021uncertainty,bai2022pessimistic,ghasemipour2022so,yan2023efficacy}, and policy-based strategies \citep{rezaeifar2022offline,xie2021bellman,zanette2021provable,nguyen2024sample}. \citep{yarats2022don} has observed that with sufficient offline data diversity and coverage, the need for pessimism to mitigate extrapolation errors and distribution shift might be reduced. To the best of our knowledge, we are the first to theoretically study the generalization ability of offline RL in the contextual MDP setting. 

\noindent\textbf{Generalization in online RL} There are extensive empirical studies on training online RL agents that can generalize to new transition and reward functions~\citep{Rajeswaran2017TowardsGA, Machado2018RevisitingTA, Justesen2018IlluminatingGI, Packer2018AssessingGI, Zhang2018ADO, Zhang2018ASO, Nichol2018GottaLF, cobbe2018quantifying,  kuettler2020nethack,  Bengio2020InterferenceAG, Bertrn2020InstanceBG, ghosh2021generalization,  kirk2021generalisation, obstacletower, ajay2021understanding, samvelyan2021minihack, frans2022powderworld, albrecht2022avalon, ehrenberg2022study, Song2020ObservationalOI,lyle2022learning,ye2020rotation, lee2020network,jiang2022uncertainty}. They  use techniques including implicit regularization \citep{Song2020ObservationalOI}, data augmentation \cite{ye2020rotation, lee2020network}, uncertainty-driven exploration~\citep{jiang2022uncertainty}, successor feature \citep{touati2023does}, etc. These works focus mostly on the online RL setting and do not provide theoretical guarantees, thus differing a lot from ours. Moreover, \cite{touati2023does} has studied zero-shot generalization in offline RL, but to unseen reward functions rather than unseen environments. 


There are also some recent works aimed at understanding online RL generalization from a theoretical perspective. \citet{wang2019generalization} examined a specific class of reparameterizable RL problems and derived generalization bounds using Rademacher complexity and the PAC-Bayes bound. \citet{malik2021generalizable} established lower bounds and introduced efficient algorithms that ensure a near-optimal policy for deterministic MDPs. A recent work \cite{ye2023power} studied how much pre-training can improve online RL test performance under different generalization settings. To the best of our knowledge, no previous work exists on theoretical understanding of the zero-shot generalization of offline RL.

Our paper is also related to recent works studying multi-task learning in reinforcement learning (RL)~\citep{brunskill2013sample,tirinzoni2020sequential,hu2021near,zhang2021provably,lu2021power,bose2024offline,ishfaq2024offline,zhang2023provably,lu2025towards}, which focus on transferring the knowledge learned from upstream tasks to downstream ones.  Additionally, these works typically assume that all tasks share similar transition dynamics or common representations while we do not. Meanwhile, they typically require the agent to interact with the downstream tasks, which does not fall into the ZSG regime.

\section{Preliminaries}\label{sec:setting}

\noindent\textbf{Contextual MDP} We study \emph{contextual episodic MDPs}, where each MDP $\cM_c$ is associated with a context $c \in C$ belongs to the context space $C$. Furthermore, $\cM_c = \{M_{c,h}\}_{h=1}^H$ consists of $H$ different individual MDPs, where each individual MDP $M_{c,h}:=(\cS, \cA, P_{c,h}(s'|s,a), r_{c,h}(s,a))$. Here $\cS$ denotes the state space, $\cA$ denotes the action space, $P_{c,h}$ denotes the transition function and $r_{c,h}$ denotes the reward function at stage $h$. We assume the starting state for each $\cM_c$ is the same state $x_1$. In this work, we interchangeablely use ``environment" or MDP to denote the MDP $\cM_c$ with different contexts.

\noindent\textbf{Policy and value function}
We denote the policy $\pi_h$ at stage $h$ as a mapping $\cS \rightarrow \Delta(\cA)$, which maps the current state to a distribution over the action space. We use $\pi = \{\pi_h\}_{h=1}^H$ to denote their collection. Then for any episodic MDP $\cM$, we define the value function for some policy $\pi$ as
\begin{small}
       \begin{align}
    &V_{\cM,h}^{\pi}(x):=\EE[ r_h+...+r_H|s_h = x, a_{h'}\sim \pi_{h'}, r_{h'}\sim r_{h'}(s_{h'}, a_{h'}), s_{h'+1}\sim P_{h'}(\cdot|s_{h'}, a_{h'}),~h'\geq h]\,,\notag\\
     &Q_{M,h}^{\pi}(x,a):=\EE[r_h+...+r_H|s_h = x,a_h = a, r_h\sim r_h(s_h,a_h), s_{h'}\sim P_{h'-1}(\cdot|s_{h'-1}, a_{h'-1}),  a_{h'}\sim \pi_{h'}, \notag\\
     &\quad r_{h'}\sim r_{h'}(s_{h'}, a_{h'}),~h'\geq h+1].\notag
\end{align} 
\end{small}
For any individual MDP $M$ with reward $r$ and transition dynamic $P$, we denote its Bellman operator $[\BB_{M}f](x,a)$ as $[\BB_{M}f](s,a):=\EE[r_h(s,a) + f(s')|s'\sim P(\cdot|s,a)]$. Then we have the well-known Bellman equation
\begin{small}
    \begin{align}
    &V_{\cM,h}^{\pi}(x)\notag = \la Q_{\cM,h}^\pi(x, \cdot), \pi_h(\cdot|x)\ra_{\cA},\ Q_{\cM,h}^{\pi}(x,a) = [\BB_{M_h} V_{\cM,h+1}^{\pi}](x,a).\notag
\end{align}
\end{small}

For simplicity, we use $V_{c,h}^\pi, Q_{c,h}^\pi, \BB_{c,h}$ to denote $V_{\cM_c,h}^\pi, Q_{\cM_c,h}^\pi, \BB_{M_{c,h}}$. We also use $\PP_c$ to denote $\PP_{\cM_c}$, the joint distribution of any potential objects under the $\cM_c$ episodic MDP. We would like to find the near-optimal policy $\pi^{*}$ w.r.t. expectation, i.e., $\pi^*:=\argmax_{\pi \in \Pi}\EE_{c\sim C}V_{c,1}^\pi(x_c)$, where $\Pi$ is the set of collection of Markovian policies, and with a little abuse of notation, we use $\EE_{c\sim C}$ to denote the expectation taken w.r.t. the i.i.d. sampling of context $c$ from the context space. Then our goal is to develop the \emph{generalizable RL} with small \emph{zero-shot generalization gap (ZSG gap)}, defined as follows:
\begin{small}
    \begin{align}
    \text{SubOpt}(\pi):=\EE_{c\sim C}\big[V_{c,1}^{\pi^*}(x_1)\big] - \EE_{c\sim C}\big[V_{c,1}^\pi(x_1)\big].\notag
\end{align}
\end{small}

\begin{remark}
We briefly compare generalizable RL with several related settings. Robust RL \citep{pinto2017robust} aims to find the best policy for the worst-case environment, whereas generalizable RL seeks a policy that performs well in the average-case environment. Meta-RL \citep{beck2023survey} enables few-shot adaptation to new environments, either through policy updates \citep{finn2017model} or via history-dependent policies \citep{duan2016rl}. In contrast, generalizable RL primarily focuses on the zero-shot setting. In the general POMDP framework \citep{cassandra1994acting}, agents need to maintain history-dependent policies to implicitly infer environment information, while generalizable RL aims to discover a single state-dependent policy that generalizes well across all environments.
\end{remark}

\begin{remark}
\citet{ye2023power} showed that in online RL, for a certain family of contextual MDPs, it is inherently impossible to determine an optimal policy for each individual MDP. Given that offline RL poses greater challenges than its online counterpart, this impossibility extends to finding optimal policies for each MDP in a zero-shot offline RL setting as well, which justifies our optimization objective on the ZSG gap. Moreover, \citet{ye2023power} showed that the few-shot RL is able to find the optimal policy for individual MDPs. Clearly, such a setting is stronger than ours, and the additional interactions are often hard to be satisfied in real-world practice. We leave the study of such a setting for future work.   
\end{remark}

\noindent\textbf{Offline RL data collection process}
The data collection process is as follows. An experimenter i.i.d. samples number $n$ of contextual episodic MDP $M_i$ from the context set (\emph{e.g.}, $i\sim C)$. For each episodic MDP $M_i$, the experimenter collects dataset $\cD_i:=\{(x_{i,h}^\tau, a_{i,h}^\tau, r_{i,h}^\tau)_{h=1}^H\}_{\tau=1}^{K}$ which includes $K$ trajectories. Note that the action $a_{i,h}^\tau$ selected by the experimenter can be arbitrary, and it does not need to follow a specific behavior policy \citep{jin2021pessimism}. 
We assume that $\cD_{i}$ is compliant with the episodic MDP $\cM_{i}$, which is defined as follows. 

\begin{definition}[\citep{jin2021pessimism}]\label{def:com}
    For a dataset $\cD_i:=\{(x_{i,h}^\tau, a_{i,h}^\tau, r_{i,h}^\tau)_{h=1}^H\}_{\tau=1}^{K}$, let $\PP_{\cD_i}$ be the joint distribution of the data collecting process. We say $\cD_{i}$ is compliant with episodic MDP $\cM_i$ if for any $x'\in \cS, r', \tau\in[K],h\in[H]$, we have
    \begin{align}
        &\PP_{\cD_i}(r_{i,h}^\tau = r', x_{i,h+1}^\tau = x'|\{(x_{i,h}^j, a_{i,h}^j)\}_{j=1}^\tau,\{(r_{i,h}^j, x_{i,h+1}^j)\}_{j=1}^{\tau-1}) \notag \\
        &\quad = \PP_{i}(r_{i,h}(s_h,a_h)=r',s_{h+1} = x'|s_h = x_h^\tau, a_h = a_h^\tau).\notag
    \end{align}
\end{definition}


In general, we claim $\cD_{i}$ is compliant with $\cM_{i}$ when the conditional distribution of any tuple of reward and next state in $\cD_i$ follows the conditional distribution determined by MDP $\cM_i$.

\section{Offline RL without context indicator information}\label{sec:nocontext}


In this section, we show that directly applying existing offline RL algorithms over datasets from multiple environments \emph{without} maintaining their identity information cannot yield a sufficient ZSG property, which is aligned with the existing observation of the poor generalization performance of offline RL \citep{mediratta2023generalization}. 


In detail, given contextual MDPs $\cM_1,...,\cM_n$ and their corresponding offline datasets $\cD_1, ..., \cD_n$, we assume the agent only has the access to the offline dataset $\bar\cD = \cup_{i=1}^n \cD_i$, where $\bar\cD = \{(x_{c_\tau, h}^\tau, a_{c_\tau, h}^\tau, r_{c_\tau, h}^\tau)_{h=1}^H\}_{\tau = 1}^{K}.$ Here $c_\tau \in C$ is the context information of trajectory $\tau$, which is \emph{unknown} to the agent. To explain why offline RL without knowing context information performs worse, we have the following proposition suggesting the offline dataset from multiple MDPs is not distinguishable from an ``average MDP" if the offline dataset does not contain context information.
\begin{proposition}\label{thm: nodistinguish}
    $\bar\cD$ is compliant with \emph{average MDP} $\bar\cM:=\{\bar M_h\}_{h=1}^H$, $\bar{M}_h:=\big(\cS,\cA,H,\bar P_h,\bar r_h\big)$,
\begin{align}    
&\bar P_h(x'|x,a)
        :=\EE_{c \sim C} \frac{P_{c,h}(x'|x,a) \mu_{c,h}(x, a)}{\EE_{c \sim C} \mu_{c,h}(x, a)},\ \PP(\bar r_h = r|x,a) := \EE_{c \sim C} \frac{\PP(\bar r_{c,h} = r|x,a) \mu_{c,h}(x, a)}{\EE_{c \sim C} \mu_{c,h}(x, a)},\notag
\end{align}
where $\mu_{c,h}(\cdot, \cdot)$ is the data collection distribution of $(s,a)$ at stage $h$ in dataset $\cD_c$. 
\end{proposition}

\begin{proof}
    See Appendix \ref{app:nodistin}.
\end{proof}


Proposition \ref{thm: nodistinguish} suggests that if no context information is revealed, then the merged offline dataset $\bar\cD$ is equivalent to a dataset collected from the average MDP $\bar \cM$. Therefore, for any offline RL which outputs a Markovian policy, it converges to the optimal policy $\bar\pi^*$ of the average MDP $\bar \cM$. 


In general, $\bar\pi^*$ can be very different from $\pi^*$  when the transition probability functions of each environment are different. For example, consider the 2-context cMDP problem shown in Figure \ref{fig:eg2}, each context consists of one state and three possible actions. The offline dataset distributions $\mu$ are marked on the arrows that both of the distributions are following near-optimal policy. By Proposition \ref{thm: nodistinguish}, in average MDP $\bar{\mathcal{M}}$ the reward of the middle action is deterministically 0, while both upper and lower actions are deterministically 1. As a result, the optimal policy $\bar{\pi}^\ast$ will only have positive probabilities toward upper and lower actions. This leads to $\mathbb{E}_{c\sim C}[V^{\overline{\pi}^\ast}_{c,1}(x_1)]=0$, though we can see that $\pi^\ast$ is deterministically choosing the middle action and $\mathbb{E}_{c\sim C}[V^{\pi^\ast}_{c,1}(x_1)]=0.5$. This theoretically illustrates that the generalization ability of offline RL algorithms without leveraging context information is weak. In sharp contrast, imitation learning such as behavior cloning (BC) converges to the teacher policy that is independent of the specific MDP. Therefore, offline RL methods such as CQL \citep{kumar2020conservative} might enjoy worse generalization performance compared with BC, which aligns with the observation made by \citet{mediratta2023generalization}.

\begin{figure}[H]
    \centering
    
    \begin{tikzpicture}[>=Stealth,
        every node/.style={font=\footnotesize},
        arrowstyle/.style={draw=none, midway, sloped, above},
        belowarrowstyle/.style={draw=none, midway, sloped, below},
        scale=0.9, transform shape]
        
        \node[draw,circle,inner sep=2pt] (x1_left) at (0,0) {$x_1$};
        \draw[->] (x1_left) -- ++(3,0.8) node[arrowstyle] {$\mu_v(a_1) = 1 - \epsilon$} node[anchor=west] {$r_v(a_1) = 1$};
        \draw[->] (x1_left) -- ++(3,0)   node[arrowstyle] {$~~~~~~~~~~~~~~\mu_v(a_2) = \epsilon$} node[anchor=west] {$r_v(a_2) = 0$};
        \draw[->] (x1_left) -- ++(3,-0.8) node[belowarrowstyle] {$\mu_v(a_3) = 0$} node[anchor=west] {$r_v(a_3) = -1$};
    \end{tikzpicture}
    \hspace{2em} 
    \begin{tikzpicture}[>=Stealth,
        every node/.style={font=\footnotesize},
        arrowstyle/.style={draw=none, midway, sloped, above},
        belowarrowstyle/.style={draw=none, midway, sloped, below},
        scale=0.9, transform shape]
        
        \node[draw,circle,inner sep=2pt] (x1_right) at (0,0) {$x_1$};
        \draw[->] (x1_right) -- ++(3,0.8) node[arrowstyle] {$\mu_w(a_1) = 0$} node[anchor=west] {$r_w(a_1) = -1$};
        \draw[->] (x1_right) -- ++(3,0)   node[arrowstyle] {$~~~~~~~~~~~~~~\mu_w(a_2) = 0$} node[anchor=west] {$r_w(a_2) = 1$};
        \draw[->] (x1_right) -- ++(3,-0.8) node[belowarrowstyle] {$\mu_w(a_3) = 1$} node[anchor=west] {$r_w(a_3) = 1$};
    \end{tikzpicture}
    \caption{Two Contextual MDPs with the same compliant average MDPs. The discrete contextual space is defined as $C=\{v,w\}$ and both MDPs satisfies $\cS=\{x_1\},\cA=\{a_1,a_2,a_3\},H=1$. The data collection distributions $\mu$ and rewards $r$ for each action of each context are specified in the graph.}
    \label{fig:eg2}
\end{figure}
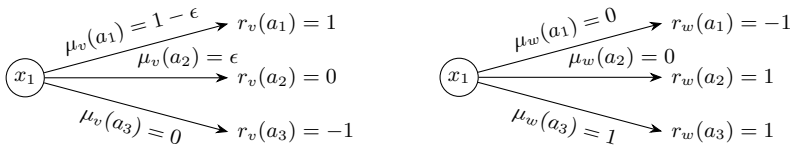

\section{Provable offline RL with zero-shot generalization}\label{sec:withcontext}

In this section, we propose offline RL with small ZSG gaps. We show that two popular offline RL approaches, \emph{model-based RL} and \emph{policy optimization-based RL}, can output RL agent with ZSG ability, with a pessimism-style modification that encourages the agent to follow the offline dataset pattern.


\subsection{Pessimistic policy evaluation}

We consider a meta-algorithm to evaluate any policy $\pi$ given an offline dataset, which serves as a key component in our proposed offline RL with ZSG. To begin with, we consider a general individual MDP and an oracle $\mathbb{O}$, which returns us an empirical Bellman operator and an uncertainty quantifier, defined as follows.

\begin{definition}[\citealt{jin2021pessimism}]\label{def:oracle}
    For any individual MDP $M $, a dataset $\cD\subseteq \cS \times \cA \times \cS \times [0,1]$ that is compliant with $M$, a test function $V_{\cD}\subseteq  [0,H]^{\cS}$ and a confidence level $\xi$, we have an oracle $\mathbb{O}(\cD, V_{\cD},\xi)$ that returns $(\hat \BB V_\cD(\cdot, \cdot), \Gamma(\cdot, \cdot))$, a tuple of Empirical Bellman operator and uncertainty quantifier, satisfying
    \begin{small}
            \begin{align}
        &\PP_{\cD}\Big(\big|(\hat\BB V_{\cD})(x,a) - (\BB_M V_\cD)(x,a)\big|\notag \leq \Gamma(x,a)~ \text{for all}~(x,a)\in \cS\times \cA   \Big) \geq 1-\xi.\notag
    \end{align}
    \end{small}
\end{definition}

\begin{remark}
    Here we adapt a test function $V_{\cD}$ that can depend on the dataset $\cD$ itself. Therefore, $\Gamma$ is a function that depends on both the dataset and the test function class. We do not specify the test function class in this definition, and we will discuss its specific realization in Section \ref{sec:linear}.
\end{remark}
\begin{remark}
    For general non-linear MDPs, one may employ the bootstrapping technique to estimate uncertainty, in line with the bootstrapped DQN approach developed by \citep{osband2016deep}. We note that when the bootstrapping method is straightforward to implement, the assumption of having access to an uncertainty quantifier is reasonable.
\end{remark}
\begin{algorithm}[t!]
\begin{small}
  \caption{\underline{P}essimistic \underline{P}olicy \underline{E}valuation (PPE)}\label{alg:model based general}
  \begin{algorithmic}[1]
    \REQUIRE Offline dataset $\{\cD_{i,h}\}_{h=1}^H$, 
    policy $\pi = (\pi_h)_{h=1}^H$, confidence probability $\delta\in(0,1)$.
    \STATE Initialize $\hat{V}^\pi_{i,H+1}(\cdot) \leftarrow 0, \ \forall i\in[n]$.
    \FOR{step $h=H, H-1, \hdots, 1$}      
        \STATE  Let $(\hB_{i,h}\hat{V}_{i,h+1}^\pi)(\cdot,\cdot),\Gamma_{i,h} (\cdot,\cdot)\leftarrow \mathbb{O}(\cD_{i,h},\hat{V}_{i,h+1}^\pi, \delta)$
        \STATE Set $\hat{Q}_{i,h}^\pi(\cdot,\cdot)\leftarrow \min\{H-h+1, (\hB_{i,h}\hat{V}_{i,h+1}^\pi)(\cdot,\cdot)-\Gamma_{i,h}(\cdot,\cdot)\}^+$
        \STATE Set $\hat{V}_{i,h}^\pi(\cdot)\leftarrow \langle \hat{Q}_{i,h}^\pi(\cdot,\cdot),\pi_h(\cdot|\cdot)\rangle_\actions$
    \ENDFOR
    \RETURN $\hat{V}_{i,1}^\pi(\cdot),\dots,\hat{V}_{i,H}^\pi(\cdot), \hat{Q}_{i,1}^\pi(\cdot,\cdot),\dots,\hat{Q}_{i,H}^\pi(\cdot,\cdot)$.
  \end{algorithmic}
\end{small}
\end{algorithm}
Based on the oracle $\mathbb{O}$, we propose our pessimistic policy evaluation (PPE) algorithm as Algorithm \ref{alg:model based general}. In general, PPE takes a given policy $\pi$ as its input, and its goal is to evaluate the V value and Q value $\{(V_{i,h}^{\pi}, Q_{i,h}^{\pi}) \}_{h=1}^H$ of $\pi$ on MDP $\cM_i$. Since the agent is not allowed to interact with $\cM_i$, PPE evaluates the value based on the offline dataset $\{\cD_{i,h}\}_{h=1}^H$. At each stage $h$, PPE utilizes the oracle $\mathbb{O}$ and obtains the empirical Bellman operator based on $\cD_{i,h}$ as well as its uncertainty quantifier, with high probability. Then PPE applies the \emph{pessimism principle} to build the estimation of the Q function based on the empirical Bellman operator and the uncertainty quantifier. Such a principle has been widely studied and used in offline policy optimization, such as pessimistic value iteration (PEVI) \citep{jin2021pessimism}. To compare with, we use the pessimism principle in the policy evaluation problem. 

\begin{remark}
    In our framework, pessimism can indeed facilitate generalization, rather than hinder it. Specifically, we employ pessimism to construct reliable Q functions for each environment individually. This approach supports broader generalization by maintaining multiple Q-networks separately. By doing so, we ensure that each Q function is robust within its specific environment, while the collective set of Q functions enables the system to generalize across different environments.
\end{remark}

\subsection{Model-based approach: pessimistic empirical risk minimization}
Given PPE, we propose algorithms that have the ZSG ability. We first propose a pessimistic empirical risk minimization (PERM) method which is model-based and conceptually simple. The algorithm details are in Algorithm \ref{alg:erm}. In detail, for each dataset $\cD_i$ drawn from $i$-th environments, PERM builds a model using PPE to evaluate the policy $\pi$ under the environment $\cM_i$. Then PERM outputs a policy $\pi^{\text{PERM}} \in \Pi$ that maximizes the average pessimistic value, i.e., $1/n\sum_{i=1}^n \hat{V}_{i,1}^\pi (x_1)$. Our approach is inspired by the classical empirical risk minimization approach adopted in supervised learning, and the Optimistic Model-based ERM proposed in \citet{ye2023power} for online RL. Our setting is more challenging than the previous ones due to the RL setting and the offline setting, where the interaction between the agent and the environment is completely disallowed. Therefore, unlike \citet{ye2023power}, which adopted an optimism-style estimation to the policy value, we adopt a pessimism-style estimation to fight the distribution shift issue in the offline setting.

Next we propose a theoretical analysis of PERM. Denote $\mathcal{N}_\epsilon^\Pi$ as the $\epsilon$-covering number of the policy space $\Pi$ w.r.t. distance $\mathrm d(\pi^1,\pi^2) = \max_{s\in \mathcal{S}, h \in [H]} \|\pi^1_h(\cdot|s) - \pi^2_h (\cdot|s)\|_{1}$. Then we have the following theorem to provide an upper bound of the suboptimality gap of the output policy $\pi^{\text{PERM}}$. 

\begin{algorithm}[t!]
\begin{small}
  \caption{\underline{P}essimistic \underline{E}mpirical \underline{R}isk \underline{M}inimization (PERM)}\label{alg:erm}
  \begin{algorithmic}[1]
    \REQUIRE Offline dataset $\cD = \{\cD_i\}_{i=1}^n, \cD_i:=\{(x_{i,h}^\tau, a_{i,h}^\tau, r_{i,h}^\tau)_{h=1}^H\}_{\tau=1}^{K}$, policy class $\Pi$, confidence probability $\delta\in(0,1)$, a pessimistic offline policy evaluation algorithm $\textbf{Evaluation}$ as a subroutine.
    \STATE Set $\cD_{i,h} =\{(x_{i,h}^\tau, a_{i,h}^\tau, r_{i,h}^\tau, x_{i,h+1}^\tau)\}_{\tau=1}^{K} $
    \STATE$\pi^{\text{PERM}}=\argmax_{\pi\in\Pi} \frac{1}{n}\sum_{i=1}^n \hat{V}_{i,1}^\pi (x_1)$, \\where $[\hat V^\pi_{i,1}(\cdot),\cdot,\dots, \cdot] = \textbf{Evaluation}\Big(\{\cD_{i,h}\}_{h=1}^H,\pi,\delta/(3nH\mathcal{N}_{(Hn)^{-1}}^\Pi))\Big)$
\RETURN $\pi^{\text{PERM}}$.

  \end{algorithmic}
\end{small}
\end{algorithm}

\begin{theorem}\label{thm:model based regret_upper_bound_general}
Set the Evaluation subroutine in Algorithm \ref{alg:erm} as PPE (Algo.\ref{alg:model based general}). Let $\Gamma_{i,h}$ be the uncertainty quantifier returned by $\mathbb{O}$ through the PERM. Then w.p. at least $1-\delta$, the output $\pi^{\text{PERM}}$ of  Algorithm \ref{alg:erm} satisfies 
\begin{small}
    \begin{align}&\normalfont
\text{SubOpt}(\pi^{\text{PERM}})\leq \underbrace{7\sqrt{\frac{2\log(6\mathcal{N}_{(Hn)^{-1}}^\Pi/\delta)}{n}}}_{I_1: \text{Supervised learning (SL) error}} +\underbrace{\frac{2}{n}\sum_{i=1}^n\sum_{h=1}^H\E{i,\pi^*}{\Gamma_{i,h}(s_h,a_h)|s_1=x_1}}_{I_2: \text{Reinforcement learning (RL) error}}\,,
\label{eq:model based gap general}
\end{align}
\end{small}
where $\EE_{i,\pi^*}$ is w.r.t. the trajectory induced by $\pi^*$ with the transition $\cP_i$ in the underlying MDP $\cM_i$.
\end{theorem}

\begin{proof}
    See Appendix \ref{appendix model based}.
\end{proof}
\begin{remark}
The covering number $\mathcal{N}_{(Hn)^{-1}}^\Pi$ depends on the policy class $\Pi$. Without any specific assumptions, the policy class $\Pi$ that consists of all the policies  $\pi = \{\pi_h\}_{h=1}^H, \pi_h:\mathcal S \mapsto \Delta (\cA)$ and the log $\epsilon$-covering number $\log \mathcal{N}^{\Pi}_{\epsilon}=O(|\cA| |\cS| H \log(1+ |\cA|/\epsilon))$. 
\end{remark}
\begin{remark}
    The SL error can be easily improved to a distribution-dependent bound $\log \cN\cdot \text{Var}/\sqrt{n}$, where $\cN$ is the covering number term denoted in $I_1$, $\text{Var} = \max_\pi\mathbb{V}_{c \sim C}V^\pi_{c,1}(x_1)$ is the variance of the context distribution, by using a Bernstein-type concentration inequality in our proof. Therefore, for the singleton environment case where $|C|=1$, our suboptimality gap reduces to the one of PEVI in \citet{jin2021pessimism}.
\end{remark}
\begin{remark}\label{rmk:merge}
    In real-world settings, as the number of sampled contexts $n$ may be very large, it is unrealistic to manage $n$ models simultaneously in the implementation of PERM algorithm, thus we provide the suboptimality bound in line with Theorem \ref{thm:model based regret_upper_bound_general} when the offline dataset is merged into $m$ contexts such that $m<n$. See Theorem \ref{thm:permv} in Appendix \ref{appendix:merge}.
\end{remark}

Theorem \ref{thm:model based regret_upper_bound_general} shows that the ZSG gap of PERM is bounded by two terms $I_1$ and $I_2$. $I_1$, which we call \emph{supervised learning error}, depends on the number of environments $n$ in the offline dataset $\cD$ and the covering number of the function (policy) class, which is similar to the generalization error in supervised learning. $I_2$, which we call it \emph{reinforcement learning error}, is decided by the optimal policy $\pi^*$ that achieves the best zero-shot generalization performance and the uncertainty quantifier $\Gamma_{i,h}$. In general, $I_2$ is the ``intrinsic uncertainty" denoted by \citet{jin2021pessimism} over $n$ MDPs, which characterizes how well each dataset $\cD_i$ covers the optimal policy $\pi^*$.

\subsection{Model-free approach: pessimistic proximal policy optimization}

\begin{algorithm}[t!]
\begin{small}
  \caption{\underline{P}essimistic \underline{P}roximal \underline{P}olicy \underline{O}ptimzation (PPPO)}\label{alg:modelfree}
  \begin{algorithmic}[1]
    \REQUIRE Offline dataset $\cD = \{\cD_i\}_{i=1}^n, \cD_i:=\{(x_{i,h}^\tau, a_{i,h}^\tau, r_{i,h}^\tau)_{h=1}^H\}_{\tau=1}^{K}$, confidence probability $\delta\in(0,1)$, a pessimistic offline policy evaluation algorithm $\textbf{Evaluation}$ as a subroutine.
    \STATE Set $\cD_{i,h} =\{(x_{i,h}^{\tau\cdot H+h}, a_{i,h}^{\tau\cdot H+h}, r_{i,h}^{\tau\cdot H+h}, x_{i,h+1}^{\tau\cdot H+h})\}_{\tau=0}^{\lfloor K/H\rfloor-1 }$
    \STATE Set $\pi_{0,h}(\cdot|\cdot)$ as uniform distribution over $\actions$ and $\hat Q^{\pi_0}_{0,h}(\cdot,\cdot)$ as zero functions.
    \FOR{$i=1, 2, \cdots, n$}
    \STATE Set $\pi_{i,h}(\cdot|\cdot)\propto \pi_{i-1,h}(\cdot|\cdot)\cdot \exp(\alpha\cdot \hat{Q}^{\pi_{i-1}}_{i-1,h}(\cdot, \cdot))$
    \STATE Set $[\cdot,\dots, \cdot, \hat Q^{\pi_i}_{i,1}(\cdot,\cdot),\dots, \hat Q^{\pi_i}_{i,H}(\cdot,\cdot)] = \textbf{Evaluation}(\{\cD_{i,h}\}_{h=1}^H,\pi_i,\delta/(nH))$
    \ENDFOR
    \RETURN $\pi^{\text{PPPO}}=\text{random}(\pi_1, ..., \pi_n)$
  \end{algorithmic}
\end{small}
\end{algorithm}



PERM in Algorithm \ref{alg:erm} works as a general model-based algorithm framework to enable ZSG for any pessimistic policy evaluation oracle. However, note that in order to implement PERM, one needs to maintain $n$ different models or critic functions simultaneously in order to evaluate $\sum_{i=1}^n \hat{V}_{i,1}^\pi (x_1)$ for any candidate policy $\pi$. Note that existing online RL \citep{ghosh2021generalization} achieves ZSG by a model-free approach, which only maintains $n$ policies rather than models/critic functions. Therefore, one natural question is whether we can design a \emph{model-free} offline RL algorithm also with access only to policies.

We propose the pessimistic proximal policy optimization (PPPO) in Algorithm \ref{alg:modelfree} to address this issue. Our algorithm is inspired by the optimistic PPO \citep{cai2020provably} originally proposed for online RL. PPPO also adapts PPE as its subroutine to evaluate any given policy pessimistically. Unlike PERM, PPPO only maintains $n$ policies $\pi_1, ...,\pi_n$, each of them is associated with an MDP $\cM_n$ from the offline dataset. In detail, PPPO assigns an order for MDPs in the offline dataset and names them $\cM_1, ..., \cM_n$. For $i$-th MDP $\cM_i$, PPPO selects the $i$-th policy $\pi_i$ as the solution of the proximal policy optimization starting from $\pi_{i-1}$, which is
\begin{align}
    \pi_{i}&\leftarrow \argmax_{\pi}V_{i-1,1}^\pi(x_1) - \alpha^{-1}\EE_{i-1, \pi_{i-1}}[\text{KL}(\pi\|\pi_{i-1})|s_1 = x_1],\label{ooo}
\end{align}
where $\alpha$ is the step size parameter. Since $V_{i-1,1}^\pi(x_1)$ is not achievable, we use a linear approximation $L_{i-1}(\pi)$ to replace $V_{i-1,1}^\pi(x_1)$, where
\begin{small}
    \begin{align}
   & L_{i-1}(\pi) = V_{i-1,1}^{\pi_{i-1}}(x_1) + \EE_{i-1, \pi_{i-1}}\bigg[\sum_{h=1}^H \la\hat Q_{i-1,h}^{\pi_{i-1}}(x_h,\cdot), \pi_h(\cdot|x_h) - \pi_{i-1,h}(\cdot|x_h) \ra\bigg|s_1 = x_1\bigg],\label{ggg}
\end{align}
\end{small}
where $\hat Q_{i-1,h}^{\pi_{i-1}}\approx Q_{i-1,h}^{\pi_{i-1}}$ are the Q values evaluated on the offline dataset for $\cM_{i-1}$. \eqref{ooo} and \eqref{ggg} give us a close-form solution of $\pi$ in Line 4 in Algorithm \ref{alg:modelfree}. Such a routine corresponds to one iteration of PPO \citep{schulman2017proximal}. Finally, PPPO outputs $\pi^{\text{PPPO}}$ as a random selection from $\pi_1,...,\pi_n$. 

\begin{remark}
    In Algorithm \ref{alg:modelfree}, we adopt a data-splitting trick \citep{jin2021pessimism} to build $\cD_{i,h}$, where we only utilize each trajectory once for one data tuple at some stage $h$. It is only used to avoid the statistical dependency of $\hat V_{i,h+1}^{\pi_i}(\cdot)$ and $x_{i,h+1}^\tau$ for the purpose of theoretical analysis. 

\end{remark}

The following theorem bounds the suboptimality of PPPO.
\begin{theorem}\label{thm:model-free}
    Set the Evaluation subroutine in Algorithm \ref{alg:modelfree} as Algorithm \ref{alg:model based general}. Let $\Gamma_{i,h}$ be the uncertainty quantifier returned by $\mathbb{O}$ through the PPPO. Selecting $\alpha = 1/\sqrt{H^2n}$. Then selecting $\delta = 1/8$, w.p. at least $2/3$, we have
    \begin{small}
        \begin{align}
    &\text{SubOpt}(\pi^{\text{PPPO}}) \leq 10\bigg(\underbrace{\sqrt{\frac{\log|\actions|H^2}{n}}}_{I_1: \text{SL error}} + \underbrace{\frac{1}{n}\sum_{i=1}^n \sum_{h=1}^H\E{i,\pi^*}{\Gamma_{i,h}(s_h,a_h)|s_1=x_1}}_{I_2: \text{RL error}}\bigg).\notag
\end{align}
    \end{small}
where $\EE_{i,\pi^*}$ is w.r.t. the trajectory induced by $\pi^*$ with the transition $\cP_i$ in the underlying MDP $\cM_i$.
\end{theorem}
\begin{proof}
    See Appendix \ref{proof:modelfree}.
\end{proof}
\begin{remark} \label{rmk:mergefree}
    As in Remark \ref{rmk:merge}, we also provide the suboptimality bound in line with Theorem \ref{thm:model-free} when the offline dataset is merged into $m$ contexts such that $m<n$. See Theorem \ref{thm:pppov} in Appendix \ref{appendix:merge}.
\end{remark}
Theorem \ref{thm:model-free} shows that the suboptimality gap of PPPO can also be bounded by the SL error $I_1$ and RL error $I_2$. Interestingly, $I_1$ in Theorem \ref{thm:model-free} for PPPO only depends on the cardinality of the action space $|\cA|$, which is different from the covering number term in $I_1$ for PERM. Such a difference is due to the fact that PPPO outputs the final policy $\pi^{\text{PPPO}}$ as a random selection from $n$ existing policies, while PERM outputs one policy $\pi^{\text{PERM}}$. Whether these two guarantees can be unified into one remains an open question. 


\section{Provable generalization for offline linear MDPs}\label{sec:linear}

In this section, we instantiate our Algo.\ref{alg:erm} and Algo.\ref{alg:modelfree} for general MDPs on specific MDP classes. We consider the linear MDPs defined as follows. 

\begin{assumption}[\citealt{yang2019sample, jin2019provably}]
We assume $\forall i \in C, \cM_i$ is a linear MDP with a known feature map $\phi:\cS\times \cA\to \RR^d$ if there exist $d$ unknown measures ${\mu}_{i,h}=(\mu_{i,h}^{(1)},\ldots,\mu_{i,h}^{(d)})$ over $\cS$ and an unknown vector $\theta_{i,h}\in \RR^d$ such that
\begin{align}
    &P_{i,h}(x'\given x,a) = \langle \phi(x,a),\mu_{i,h}(x')\rangle, \EE\bigl[r_{i,h}(s_h, a_h) \biggiven s_h=x,a_h=a\bigr] = \langle \phi(x,a),\theta_{i,h}\rangle
\end{align}\label{eq:w07}
for all $(x,a,x')\in \cS\times \cA\times \cS$ at every step $h\in[H]$. We assume $\|\phi(x,a)\|\leq 1$ for all $(x,a)\in \cS\times \cA$ and $\max\{\|\mu_{i,h}(\cS) \| ,\|\theta_{i,h}\|\}\leq \sqrt{d}$ at each step $h\in[H]$, and we define  $\| \mu_{i,h} (\cS) \| = \int_{\cS } \| \mu_{i,h} (x) \| \,\ud x$.
\label{assump:linear_mdp}
\end{assumption}


We first specialize the general PPE algorithm (Algo.\ref{alg:model based general}) to obtain the PPE algorithm tailored for linear MDPs (Algo.\ref{alg:linear mdp model based}). This specialization is achieved by constructing $\hat\BB_{i,h}\hat{V}^\pi_{i,h+1}$, $\Gamma_{i,h}$, and $\hat{V}^\pi_{i,h}$ based on the dataset $\cD_i$. We denote the set of trajectory indexes in $\cD_{i,h}$ as $\cB_{i,h}$. Algo.\ref{alg:linear mdp model based} subsequently functions as the policy evaluation subroutine in Algo.\ref{alg:erm} and Algo.\ref{alg:modelfree} for linear MDPs. In detail, we construct $\hat\BB_{i,h}\hat{V}_{i,h+1}$ (which is the estimation of $\BB_{i,h}\hat{V}_{i,h+1}$) as $(\hat\BB_{i,h}\hat{V}_{i,h+1})(x, a) = \phi(x, a)^\top \hat{w}_{i,h}$,
where
\begin{align}
    &\textstyle{\hat{w}_{i,h} =  \argmin_{w\in \RR^d} \sum_{\tau \in \cB_{i,h}}} \bigl(r_{i,h}^\tau + \hat{V}_{i,h+1}(x_{i,h}^{-,\tau})  - \phi (x_{i,h}^\tau,a_{i,h}^\tau)^\top w\bigr)^2 + \lambda \cdot \|w\|_2^2\,\label{eq:w188}
\end{align}
with $\lambda>0$ being the regularization parameter. The closed-form solution to \eqref{eq:w188} is in Line 4 in Algorithm \ref{alg:linear mdp model based}. Besides, we construct the uncertainty quantifier $\Gamma_{i,h}$ based on $\cD_i$ as 
\begin{align}
    \textstyle{\Gamma_{i,h}(x, a)}& = \beta(\delta)\cdot\|\phi(x, a)\|_{\Lambda_{i,h} ^{-1}}\,,\Lambda_{i,h} = \sum_{\tau \in \cB_{i,h}}\phi(x_{i,h}^\tau,a_{i,h}^\tau)  \phi(x_{i,h}^\tau,a_{i,h}^\tau) ^\top + \lambda\cdot I,\notag
\end{align}  
with $\beta(\delta)>0$ being the scaling parameter.


\begin{algorithm}[H]
\begin{small}
\caption{\underline{P}essimistic \underline{P}olicy \underline{E}valuation (PPE): Linear MDP}\label{alg:linear mdp model based}
\begin{algorithmic}[1]
\REQUIRE Offline dataset $\{\cD_{i,h}\}_{h=1}^H, \cD_{i,h}=\{(x_{i,h}^\tau,a_{i,h}^\tau,r_{i,h}^\tau, x_{i,h}^{-,\tau})\}_{\tau \in \cB_{i,h}}$, policy $\pi$, confidence probability $\delta\in(0,1)$.
    \STATE Initialize $\hat{V}^\pi_{i,H+1}(\cdot) \leftarrow 0, \ \forall i\in[n]$.
\FOR{step $h=H,H-1,\ldots,1$}
\STATE Set $\Lambda_{i,h} \leftarrow \sum_{\tau \in \cB_{i,h}} \phi(x_{i,h}^\tau,a_{i,h}^\tau)  \phi(x_{i,h}^\tau,a_{i,h}^\tau) ^\top + \lambda\cdot I$. 
\STATE Set $\hat{w}_{i,h}\leftarrow  \Lambda_{i,h} ^{-1}( \sum_{\tau \in \cB_{i,h}} \phi(x_{i,h}^\tau,a_{i,h}^\tau) \cdot (r_{i,h}^\tau + \hat{V}_{i,h+1}^\pi(x_{i,h}^{-,\tau})) ) $. 
\STATE Set $\Gamma_{i,h}(\cdot,\cdot) \leftarrow \beta(\delta)\cdot ( \phi(\cdot,\cdot)^\top  \Lambda_{i,h} ^{-1} \phi(\cdot,\cdot) )^{1/2}$. 
\STATE Set $\hat{Q}_{i,h}^\pi(\cdot,\cdot) \leftarrow \min\{\phi(\cdot,\cdot)^\top \hat{w}_{i,h} - \Gamma_{i,h}(\cdot,\cdot),H-h+1\}^+$. 
        \STATE Set $\hat{V}_{i,h}^\pi(\cdot)\leftarrow \langle \hat{Q}_{i,h}^\pi(\cdot,\cdot),\pi_h(\cdot|\cdot)\rangle_\actions$
\ENDFOR 
    \RETURN $\hat{V}_{i,1}^\pi(\cdot),\dots,\hat{V}_{i,H}^\pi(\cdot), \hat{Q}_{i,1}^\pi(\cdot,\cdot),\dots,\hat{Q}_{i,H}^\pi(\cdot,\cdot)$.
\end{algorithmic}
    
\end{small}
\end{algorithm}
The following theorem shows the suboptimality gaps for Algo.\ref{alg:erm} (utilizing subroutine Algo.\ref{alg:linear mdp model based}) and Algo.\ref{alg:modelfree} (also with subroutine Algo.\ref{alg:linear mdp model based}).

\begin{theorem}\label{thm:regret_upper_linear}
Under Assumption \ref{assump:linear_mdp}, in Algorithm \ref{alg:linear mdp model based}, we set $\lambda=1,\quad \beta(\delta) = c\cdot dH\sqrt{\log(2dHK/\delta)}$, where $c>0$ is a positive constant. Then, we have: \\
(i) for the output policy $\pi^{\text{PERM}}$ of Algo.\ref{alg:erm} with subroutine Algo.\ref{alg:linear mdp model based}, w.p. at least $1-\delta$, the suboptimality gap satisfies
\begin{small}
    \begin{align}&
\text{SubOpt}(\pi^{\text{PERM}})\leq 7\sqrt{\frac{7\log(6\mathcal{N}_{(Hn)^{-1}}^\Pi/\delta)}{n}}+\frac{2\beta\big(\frac{\delta}{3nH\mathcal{N}_{(Hn)^{-1}}^\Pi}\big)}{n}\cdot\sum_{i=1}^n\sum_{h=1}^H \EE_{i,\pi^*}\Bigl[ \|\phi(s_h,a_h)\|_{\tilde\Lambda_{i,h}^{-1}} \biggiven s_1=x_1\Bigr]\,,
\label{eq:model based gap linear}
\end{align}
\end{small}
(ii) for the output policy $\pi^{\text{PPPO}}$ of Algo.\ref{alg:modelfree} with subroutine Algo.\ref{alg:linear mdp model based}, setting $\delta = 1/8$, then with probability at least $2/3$, the suboptimality gap satisfies
\begin{small}
    \begin{align}&\normalfont
\text{SubOpt}(\pi^{\text{PPPO}})\leq 10\bigg(\sqrt{\frac{\log|\actions|H^2}{n}}+\frac{\beta\big(\frac{1}{4nH}\big)}{n}\cdot\sum_{i=1}^n\sum_{h=1}^H \EE_{i,\pi^*}\Bigl[ \|\phi(s_h,a_h)\|_{\bar\Lambda_{i,h}^{-1}} \biggiven s_1=x_1\Bigr]\bigg),
\label{eq:modelfree gap linear}
\end{align}
\end{small}
where $\EE_{i,\pi^*}$ is with respect to the trajectory induced by $\pi^*$ with the transition $\cP_i$ in the underlying MDP $\cM_i$ given the fixed matrix $\tilde\Lambda_{i,h}$ or $\bar\Lambda_{i,h}$.
\end{theorem}

$\|\phi(s_h,a_h)\|_{\Lambda_{i,h}^{-1}}$ indicates how well the state-action pair $(s_h,a_h)$ is covered by the dataset $\cD_i$. The term $\sum_{i=1}^n\sum_{h=1}^H \EE_{i,\pi^*}\Bigl[ \|\phi(s_h,a_h)\|_{\Lambda_{i,h}^{-1}} \biggiven s_1=x_1\Bigr]$ in the suboptimality gap in Theorem \ref{thm:regret_upper_linear} is small if for each context $i\in[n]$, the dataset $\cD_i$ well covers the trajectory induced by the optimal policy $\pi^*$ on the corresponding MDP $\cM_i$.


\noindent\textbf{Well-explored behavior policy} Next we consider a case where the dataset $\cD$ consists of i.i.d. trajectories collecting from different environments. Suppose $\cD$ consists of $n$ independent datasets $\cD_1,\ldots,\cD_n$, and for each environment $i$, $\cD_i$ consists of $K$ trajectories $\cD_i = \{(x_{i,h}^\tau, a_{i,h}^\tau, r_{i,h}^\tau)_{h=1}^H\}_{\tau=1}^{K}$ independently and identically induced by a fixed behavior policy $\bar\pi_i$ in the linear MDP $\cM_i$. We have the following assumption on well-explored policy:
\begin{definition}[\citealt{duan2020minimax, jin2021pessimism}]\label{ass:wellexp}
     For an behavior policy $\bar\pi$ and an episodic linear MDP $\cM$ with feature map $\phi$, we say $\bar\pi$ well-explores $\cM$ with constant $c$ if there exists an absolute positive constant $c > 0$ such that 
     \begin{small}
         \begin{align*}
   &\forall h \in [H], \lambda_{\min}(\Sigma_{h})\geq c/d, \text{where~~} \Sigma_{h} = \EE_{\bar\pi, \cM}\bigl[\phi(s_h,a_h)\phi(s_h,a_h)^\top\bigr].\notag
\end{align*}
     \end{small}
\end{definition}
A well-explored policy guarantees that the obtained trajectories is ``uniform" enough to represent any policy and value function. The following corollary shows that with the above assumption, the suboptimality gaps of Algo.\ref{alg:erm} (with subroutine Algo.\ref{alg:linear mdp model based}) and Algo.\ref{alg:modelfree} (with subroutine Algo.\ref{alg:linear mdp model based}) decay to 0 when $n$ and $K$ are large enough.

\begin{corollary}\label{cor:well_explore}
Suppose that for each $i\in[n]$, $\cD_i$ is generated by behavior policy $\bar\pi_i$ which well-explores MDP $\cM_i$ with constant $c_i\geq c_{\text{min}}$. In Algo.\ref{alg:linear mdp model based}, we set $\lambda=1,\beta(\delta) = c'\cdot dH\sqrt{\log(4dHK/\delta)}$ where $c' >0$ is a positive constant. 
Suppose we have 
$K \geq 40d/c_{\text{min}}\log (4 dnH/ \delta)$ and set $C_n^*:=1/n\cdot \sum_{i=1}^n c_i^{-1/2}$. Then we have: \\
(i) for the output $\pi^{\text{PERM}}$ of Algo.\ref{alg:erm} with subroutine Algo.\ref{alg:linear mdp model based}, w.p. at least $1-\delta$, the suboptimality gap satisfies 
\begin{small}
    \begin{align}
&\text{SubOpt}(\pi^{\text{PERM}})\leq 7\sqrt{\frac{2\log(6\mathcal{N}_{(Hn)^{-1}}^\Pi/\delta)}{n}}+2\sqrt{2} c'\cdot d^{3/2} H^2 K^{-1/2} \sqrt{\log(12dHnK\mathcal{N}_{(Hn)^{-1}}^\Pi/\delta)}\cdot C_n^*\,,
\label{eq:event_opt_explore_d model based}
\end{align}
\end{small}
(ii) for the output policy $\pi^{\text{PPPO}}$ of Algo.\ref{alg:modelfree} with subroutine Algo.\ref{alg:linear mdp model based}, setting $\delta = 1/8$, then with probability at least $2/3$, the suboptimality gap satisfies
\begin{small}
    \begin{align}&\normalfont
\text{SubOpt}(\pi^{\text{PPPO}})\leq 10\bigg(\sqrt{\frac{\log|\actions|H^2}{n}}+2\sqrt{2} c'\cdot d^{3/2} H^{2.5} K^{-1/2} \sqrt{\log(16dHnK)}\cdot C_n^*\bigg).
\label{eq:event_opt_explore_d model free}
\end{align}
\end{small}
\end{corollary}

\begin{remark}
    The mixed coverage parameter $C_n^*=\frac{1}{n}\sum_{i=1}^n\frac{1}{\sqrt{c_i}}$ is small if for any $i\in[n]$, $c_i$ is large, \emph{i.e.}, the minimum eigenvalue of $\Sigma_{i,h}=\EE_{\bar\pi_i, \cM_i}\bigl[\phi(s_h,a_h)\phi(s_h,a_h)^\top\bigr]$ is large. Note that $\lambda_{\text{min}}(\Sigma_{i,h})$ indicates how well the behavior policy $\bar\pi_i$ explores the state-action pairs on MDP $\cM_i$; this shows that if for each environment $i\in[n]$, the behavior policy explores $\cM_i$ well, the suboptimality gap will be small.
\end{remark}
\begin{remark}
    Under the same conditions of Corollary \ref{cor:well_explore}:\\
    (i) If $n\geq\frac{392\log(6\mathcal{N}_{(Hn)^{-1}}^\Pi/\delta)}{\epsilon^2}$ and $K\geq\max\{\frac{40d}{c_{\text{min}}}\log ( \frac{4dnH}{\delta}),\frac{32c'^2d^3H^4\log(12dHnK\mathcal{N}_{(Hn)^{-1}}^\Pi/\delta)C_n^{*2}}{\epsilon^2}\}$, then w.p. at least $1-\delta$, $\text{SubOpt}(\pi^{\text{PERM}})\leq \epsilon$.
    \\(ii)   If $n\geq\frac{400H^2\log(|\actions|)}{\epsilon^2}$ and $K\geq\max\{\frac{40d}{c_{\text{min}}}\log ( 16dnH),\frac{32c'^2d^3H^5\log(16dHnK)C_n^{*2}}{\epsilon^2}\}$, then w.p. at least $2/3$, $\text{SubOpt}(\pi^{\text{PPPO}})\leq \epsilon$.
\end{remark}
Corollary \ref{cor:well_explore} suggests that both of our proposed algorithms enjoy the $O(n^{-1/2} + K^{-1/2}\cdot C_n^*)$ convergence rate to the optimal policy $\pi^*$ given a well-exploration data collection assumption, where $C_n^*$ is a mixed coverage parameter over $n$ environments defined in Corollary \ref{cor:well_explore}.

\section{Conclusion and Future Work}\label{sec: conclusion}
In this work, we study the zero-shot generalization (ZSG) performance of offline reinforcement learning (RL). We propose two offline RL frameworks, pessimistic empirical risk minimization and pessimistic proximal policy optimization, and show that both of them can find the optimal policy with ZSG ability. We also show that such a generalization property does not hold for offline RL without knowing the context information of the environment, which demonstrates the necessity of our proposed new algorithms. Currently, our theorems and algorithm design depend on the i.i.d. assumption of the environment selection. How to relax such an assumption remains an interesting future direction.



\appendix

\begin{center}
    \LARGE \textbf{Appendix}
    \rule{\linewidth}{0.8pt}
\end{center}

We provide missing proofs and theoretical results of our paper in the Appendix sections:
\begin{itemize}[leftmargin = *]
   \item In Appendix \ref{app:notext}, we provide the missing results of Section \ref{sec:nocontext}. We first provide the proof of Proposition \ref{thm: nodistinguish}, then we analyze the suboptimality gap of the Pessimistic Value Iteration (PEVI) (\cite{jin2021pessimism}) in the contextual linear MDP setting without context information.
   \item In Appendix \ref{app:mainthm}, we provide the proofs of our main theorems on the suboptimality bounds of PERM and PPPO in Section \ref{sec:withcontext}.
   \item In Appendix \ref{appendix:merge}, we state and prove the suboptimality bounds we promised in Remarks \ref{rmk:merge} and \ref{rmk:mergefree}, where we merge the sampled contexts into $m$ groups ($m<n$) to reduce the computational complexity in practical settings.
   \item In Appendix \ref{proof:linear}, we provide the proofs of results in Section \ref{sec:linear} on linear MDPs. Namely, we provide proof of Theorem \ref{thm:regret_upper_linear}, proof of Corollary \ref{cor:well_explore}.
\end{itemize}

\section{Results in Section \ref{sec:nocontext}}\label{app:notext}

\subsection{Proof of Proposition \ref{thm: nodistinguish}}\label{app:nodistin}

    Let $\cD' = \{(x_{c_\tau,h}^\tau, a_{c_\tau,h}^\tau, r_{c_\tau,h}^\tau)\}_{h=1, \tau = 1}^{H,K}$ denote the merged dataset, where each trajectory belongs to a context $c_\tau$. For simplicity, let $\cD_c$ denote the collection of trajectories that belong to MDP $\cM_c$.   
    Then each trajectory in $\cD'$ is generated by the following steps:
    \begin{itemize}
        \item The experimenter randomly samples an environment $c \sim C$.
        \item The experimenter collect a trajectory from the episodic MDP $\cM_c$.
    \end{itemize}
  Then for any $x', r', \tau$ we have 
    \begin{align}
        &\PP_{\cD'}(r_{c_\tau,h}^\tau = r', x_{c_\tau,h+1}^\tau = x'|\{(x_{c_j,h}^j, a_{c_j,h}^j)\}_{j=1}^{\tau}, \{r_{c_j,h}^j, x_{c_j,h+1}^j\}_{j=1}^{\tau-1})\notag \\
        & = \frac{\PP_{\cD'}(r_{c_\tau,h}^\tau = r', x_{c_\tau,h+1}^\tau = x',\{(x_{c_j,h}^j, a_{c_j,h}^j)\}_{j=1}^{\tau}, \{r_{c_j,h}^j, x_{c_j,h+1}^j\}_{j=1}^{\tau-1})}{\PP_{\cD'}(\{(x_{c_j,h}^j, a_{c_j,h}^j)\}_{j=1}^{\tau}, \{r_{c_j,h}^j, x_{c_j,h+1}^j\}_{j=1}^{\tau-1})}\notag \\
        & = \sum_{c\in C}\PP_{\cD'}(r_{c_\tau,h}^\tau = r', x_{c_\tau,h+1}^\tau = x'|\{(x_{c_j,h}^j, a_{c_j,h}^j)\}_{j=1}^{\tau}, \{r_{c_j,h}^j, x_{c_j,h+1}^j\}_{j=1}^{\tau-1}, c_\tau = c)q(c),\label{www:1}
    \end{align}
where 
\begin{align}
    q(c'):=\frac{\PP_{\cD'}(\{(x_{c_j,h}^j, a_{c_j,h}^j)\}_{j=1}^{\tau}, \{r_{c_j,h}^j, x_{c_j,h+1}^j\}_{j=1}^{\tau-1}, c_\tau = c')}{\sum_{c\in C}\PP_{\cD'}(\{(x_{c_j,h}^j, a_{c_j,h}^j)\}_{j=1}^{\tau}, \{r_{c_j,h}^j, x_{c_j,h+1}^j\}_{j=1}^{\tau-1}, c_\tau = c)}.\notag
\end{align}
Next, we further have
\begin{align}
&\eqref{www:1}\notag \\
&=\sum_{c\in C}\PP_{c}(r_{c,h}(s_h) = r', s_{h+1} = x'|s_h = x_{c_\tau, h}^\tau, a_h = a_{c_\tau, h}^\tau)q(c)\notag \\
        & =\sum_{c\in C}\frac{\PP_{c}(r_{c,h}(s_h) = r', s_{h+1} = x'|s_h = x_{c_\tau, h}^\tau, a_h = a_{c_\tau, h}^\tau)\PP_{\cD'}(s_h = x_{c_\tau, h}^\tau, a_h = a_{c_\tau, h}^\tau, c_\tau = c)}{\sum_{c\in C}\PP_{\cD'}(s_h = x_{c_\tau, h}^\tau, a_h = a_{c_\tau, h}^\tau, c_\tau = c)}\notag \\
        &=\sum_{c\in C}p(c)\cdot \frac{\PP_{c}(r_{c,h}(s_h) = r', s_{h+1} = x'|s_h = x_{c_\tau, h}^\tau, a_h = a_{c_\tau, h}^\tau)\PP_{c}(s_h = x_{c_\tau, h}^\tau, a_h = a_{c_\tau, h}^\tau)}{\sum_{c\in C}p(c)\cdot \PP_{c}(s_h = x_{c_\tau, h}^\tau, a_h = a_{c_\tau, h}^\tau)}\notag\\
        &=\EE_{c \sim C} \frac{\PP_{c}(r_{c,h}(s_h) = r', s_{h+1} = x'|s_h = x_{c_\tau, h}^\tau, a_h = a_{c_\tau, h}^\tau) \mu_{c,h}(x_{c_\tau, h}^\tau, a_{c_\tau, h}^\tau)}{\EE_{c \sim C} \mu_{c,h}(x_{c_\tau, h}^\tau, a_{c_\tau, h}^\tau)},\notag
\end{align}
where the first equality holds since for all trajectories $\tau$ satisfying $c_\tau = c$, they are compliant with $\cM_c$, the second one holds since all trajectories are independent of each other, the third and fourth ones hold due to the definition of $\mu_{c,h}(\cdot, \cdot)$. 

\subsection{PEVI algorithm}\label{app:pevi}

\begin{algorithm}[H]
\caption{\citep{jin2021pessimism} Pessimistic Value Iteration (PEVI)}
\begin{algorithmic}[1]
\label{alg:no context}
\REQUIRE Dataset $\cD=\{(x_{c_\tau, h}^\tau,a_{c_\tau, h}^\tau,r_{c_\tau, h}^\tau)_{h=1}^H\}_{\tau=1}^{K}$, confidence probability $\delta\in(0,1)$.
\STATE Initialization: Set $\hat{V}_{H+1}(\cdot) \leftarrow 0$.
\FOR{step $h=H,H-1,\ldots,1$}
\STATE Set $\Lambda_h \leftarrow \sum_{\tau=1}^K \phi(x_h^\tau,a_h^\tau)  \phi(x_h^\tau,a_h^\tau) ^\top + \lambda\cdot I$. 
\STATE Set $\hat{w}_h\leftarrow  \Lambda_h ^{-1}( \sum_{\tau=1}^{K} \phi(x_h^\tau,a_h^\tau) \cdot (r_h^\tau + \hat{V}_{h+1}(x_{h+1}^\tau)) ) $. 
\STATE Set $\Gamma_h(\cdot,\cdot) \leftarrow \beta(\delta)\cdot ( \phi(\cdot,\cdot)^\top  \Lambda_h ^{-1} \phi(\cdot,\cdot) )^{1/2}$. 
\STATE Set $\hat{Q}_h(\cdot,\cdot) \leftarrow \min\{\phi(\cdot,\cdot)^\top \hat{w}_h - \Gamma_h(\cdot,\cdot),H-h+1\}^+$. 
\STATE Set $\hat{\pi}_h (\cdot \given \cdot) \leftarrow \argmax_{\pi_h}\langle \hat{Q}_h(\cdot, \cdot),\pi_h(\cdot\given \cdot)\rangle_{\cA}$.
\STATE Set $\hat{V}_h(\cdot) \leftarrow \langle \hat{Q}_h(\cdot,\cdot),\hat\pi_h(\cdot \given \cdot)\rangle_{\cA}$.
\ENDFOR 
\RETURN $\pi^{\text{PEVI}}= \{\hat{\pi}_h\}_{h=1}^H$.
\end{algorithmic}
\end{algorithm}

We analyze the suboptimality gap of the Pessimistic Value Iteration (PEVI) (\cite{jin2021pessimism}) in the contextual linear MDP setting without context information to demonstrate that by finding the optimal policy for $\bar\cM$ is not enough to find the policy that performs well on MDPs with context information. 

\noindent\textbf{Pessimistic Value Iteration (PEVI)}.
Let $\overline{\pi}^*$ be the optimal policy w.r.t. the average MDP $\bar \cM$. We analyze the performance of the Pessimistic Value Iteration (PEVI) \citep{jin2021pessimism} under the unknown context information setting. The details of PEVI is in Algo.\ref{alg:no context}.

Suppose that $\bar\cD$ consists of $K$ number of trajectories generated i.i.d. following by a fixed behavior policy $\bar\pi$. Then the following theorem shows the suboptimality gap for Algo.\ref{alg:no context} does not converge to 0 even when the data size grows to infinity.

\begin{theorem}\label{thm:pevi}
Assume that $\bar\pi$
In Algo.\ref{alg:linear mdp model based}, we set 
\begin{equation}
    \lambda=1,\quad \beta(\delta) = c'\cdot dH\sqrt{\log(4dHK/\delta)}\,,
\end{equation}
where $c' >0$ is a positive constant. 
Suppose we have 
$K \geq \tilde c\cdot d \log (4 dH/ \xi)$, where $\tilde c > 0$ is a sufficiently large positive constant that depends on $c$. Then we have: w.p. at least $1-\delta$, for the output policy $\pi^{\text{PEVI}}$ of Algo.\ref{alg:no context},
\begin{align}
    \sup_{\pi}V_{\bar\cM,1}^\pi - V_{\bar\cM,1}^{\pi^{\text{PEVI}}} \leq c'' \cdot d^{3/2} H^2 K^{-1/2} \sqrt{\log(4dHK/\delta)}, 
\end{align}
and the suboptimality gap satisfies
\begin{align}\normalfont
\text{SubOpt}(\pi^{\text{PEVI}})  \leq c'' \cdot d^{3/2} H^2 K^{-1/2} \sqrt{\log(4dHK/\delta)}+ 2\sup_\pi |V_{\bar{\cM}, 1}^{\pi}(x_1)-\EE_{c\sim C} V_{c, 1}^{\pi}(x_1)|\,,
\label{eq:model based gap no context}
\end{align}
where $c''>0$ is a positive constant that only depends on $c$ and $c'$.
\label{cor:well_explore2}
\end{theorem}

\begin{proof}[Proof of Theorem \ref{cor:well_explore2}]
    First, we define the value function on the average MDP $\Bar{\cM}$ as follows.

\begin{equation}
    \overline{V}^\pi_{h}(x)=\EE_{\pi,\Bar{\cM}}\Big[ \sum_{i=h}^H r_i(s_i, a_i)\biggiven s_h=x  \Big]\,.
\end{equation}

We then decompose the suboptimality gap as follows.
\begin{align}\normalfont
&\text{SubOpt}(\pi^{\text{PEVI}}) \notag \\
&= \EE_{c\sim C}\big[V_{c,1}^{\pi^*}(x_1)\big] - \EE_{c\sim C}\big[V_{c,1}^{\pi^{\text{PEVI}}}(x_1)\big]\notag\\
&=\overline{V}^{\overline{\pi}^*}_{1}(x_1)-\overline{V}^{\pi^{\text{PEVI}}}_1(x_1)+\big(\EE_{c\sim C}\big[V_{c,1}^{\pi^*}(x_1)\big]-\overline{V}^{\overline{\pi}^*}_{1}(x_1)\big)+\big(\overline{V}^{\pi^{\text{PEVI}}}_1(x_1)-\EE_{c\sim C}\big[V_{c,1}^{\pi^{\text{PEVI}}}(x_1)\big]\big)\notag\\
&\leq\overline{V}^{\overline{\pi}^*}_{1}(x_1)-\overline{V}^{\pi^{\text{PEVI}}}_1(x_1)+2\sup_\pi |V_{\bar{\cM}, 1}^{\pi}(x_1)-\EE_{c\sim C} V_{c, 1}^{\pi}(x_1)|\,.
\label{eq:model based gap no context}
\end{align}

Then, applying Corollary 4.6 in \cite{jin2021pessimism}, we can get that w.p. at least $1-\delta$
\begin{equation}
    \overline{V}^{\overline{\pi}^*}_{1}(x_1)-\overline{V}^{\pi^{\text{PEVI}}}_1(x_1)\leq  c'' \cdot d^{3/2} H^2 K^{-1/2} \sqrt{\log(4dHK/\delta)}\,,
\end{equation}
which, together with Eq.(\ref{eq:model based gap no context}) completes the proof.

\end{proof}

Theorem \ref{thm:pevi} shows that by adapting the standard pessimistic offline RL algorithm over the offline dataset without context information, the learned policy $\pi^{\text{PEVI}}$ converges to the optimal policy $\bar\pi^*$ over the average MDP $\bar\cM$.

\section{Proof of Theorems in Section \ref{sec:withcontext}}\label{app:mainthm}

\subsection{Proof of Theorem \ref{thm:model based regret_upper_bound_general}}\label{appendix model based}

We define the model estimation error as 
\begin{equation}
\iota_{i,h}^\pi(x,a) = (\BB_{i,h} \hat{V}^\pi_{i,h+1})(x,a) - \hat{Q}^\pi_{i,h}(x,a).
\label{eq:def_iota with context}
\end{equation}
And we define the following condition 
\begin{equation}
    \big|(\hat\BB_{i,h} \hat{V}^\pi_{i,h+1})(x,a) - (\BB_{i,h} \hat{V}^\pi_{i,h+1})(x,a)\big|\leq \Gamma_{i,h}(x,a)~ \text{for all}~i\in[n], \pi\in\Pi, (x,a)\in \cS\times \cA, h\in[H]\,. \label{high prob 2}
\end{equation}
We introduce the following lemma to bound the model estimation error.
\begin{lemma}[Model estimation error bound (Adapted from Lemma 5.1 in \cite{jin2021pessimism}]
Under the condition of Eq.(\ref{high prob 2}), we have
\label{lem:model_eval_err}
\begin{equation}
    0\leq \iota_{i,h}^\pi(x,a) \leq 2\Gamma_{i,h}(x,a),\quad \text{for all}~~i\in[n], ~\pi\in\Pi, ~(x,a)\in \cS\times \cA,~ h\in [H]. 
\end{equation}\label{eq:model_eval_err_bound}
\end{lemma}

Then, we prove the following lemma for pessimism in V values.
\begin{lemma}[Pessimism for Estimated V Values]
   Under the condition of Eq.(\ref{high prob 2}), for any $i\in[n], \pi\in\Pi, x \in \cS$, we have
    \begin{equation}        V_{i,h}^\pi(x)\geq\hat{V}_{i,h}^\pi(x)\,.
    \end{equation}
    \label{pessimism lemma1}
\end{lemma}

\begin{proof}
    For any $i\in[n], \pi\in\Pi, x\in\cS, a\in \cA$, we have 
    \begin{align}
        &Q_{i,h}^\pi(x,a)-\hat{Q}_{i,h}^\pi(x,a)\notag \\
        &\geq r_{i,h}(x,a)+(\BB_{i,h}V^\pi_{i,h+1})(x,a)-\big(r_{i,h}(s,a)+(\hat{\BB}_{i,h}\hat{V}^\pi_{i,h+1})(x,a)-\Gamma_{i,h}(x,a)\big)\notag\\
        &=(\BB_{i,h}V^\pi_{i,h+1})(x,a)-({\BB}_{i,h}\hat{V}^\pi_{i,h+1})(x,a)+\Gamma_{i,h}(x,a)\notag\\
        &\quad-\big((\hat{\BB}_{i,h}\hat{V}^\pi_{i,h+1})(x,a)-{\BB}_{i,h}\hat{V}^\pi_{i,h+1})(x,a)\big)\notag\\
        &\geq (\BB_{i,h}V^\pi_{i,h+1})(x,a)-({\BB}_{i,h}\hat{V}^\pi_{i,h+1})(x,a)\notag\\
        &=\big(P_{i,h}(V_{i,h+1}^\pi-\hat{V}_{i,h+1}^\pi)\big)(x,a)\notag\,,
    \end{align}
    where the second inequality is because of Eq.(\ref{high prob 2}). And since in the $H+1$ step we have $V_{i,H+1}^\pi=\hat{V}_{i,h+1}^\pi=0$, we can get $Q_{i,H}^\pi(x,a)-\hat{Q}_{i,H}^\pi(x,a)$. Then we use induction to prove $Q_{i,h}^\pi(x,a)\geq\hat{Q}_{i,h}^\pi(x,a)$ for all $h$. Given $Q_{i,h+1}^\pi(x,a)\geq\hat{Q}_{i,h+1}^\pi(x,a)$, we have
    \begin{align}
        Q_{i,h}^\pi(x,a)-\hat{Q}_{i,h}^\pi(x,a)&\geq \big(P_{i,h}(V_{i,h+1}^\pi-\hat{V}_{i,h+1}^\pi)\big)(x,a)\notag\\
        &=\E{}{\langle Q_{i,h+1}^\pi(s_{h+1},\cdot)-\hat{Q}_{i,h+1}^\pi(s_{h+1},\cdot),\pi_{h+1}(\cdot|s_{h+1})\rangle_\cA|s_h=x, a_h=a}\notag\\
        &\geq 0\,.
    \end{align}
    Then we have
    \begin{align}
        V_{i,h}^\pi(x)-\hat{V}_{i,h}^\pi(x)&=\langle Q_{i,h}^\pi(x, \cdot)-\hat{Q}_{i,h}^\pi(x, \cdot),\pi_h(\cdot\given x)\rangle_{\cA}\geq 0\notag\,.
    \end{align}
\end{proof}

Then we start our proof. 
\begin{proof}[Proof of Theorem \ref{thm:model based regret_upper_bound_general}]

First, we decompose the suboptimality gap as follows
\begin{align}
    &\text{SubOpt}(\pi^{\text{PERM}})\notag \\
    &=\EE_{c \sim C}{V_{c,1}^{\pi^*}(x_1)-V_{c,1}^{\hatpistar}(x_1)}\notag \\
    &=\EE_{c \sim C}{V_{c,1}^{\pi^*}(x_1)}-\frac{1}{n}\sum_{i=1}^n V^{\pi^*}_{i,1}(x_1) +\frac{1}{n}\sum_{i=1}^n V_{i,1}^{\pi^{\text{PERM}}}(x_1)-\EE_{c \sim C}{V_{c,1}^{\pi^{\text{PERM}}}(x_1)}\notag\\
    &\quad +\frac{1}{n}\sum_{i=1}^n \big(V_{i,1}^{\pi^*}(x_1)-V_{i,1}^{\pi^{\text{PERM}}}(x_1)\big)\label{decomposition}\,.
\end{align}

For the first two terms, we can bound them following the standard generalization techniques (\cite{ye2023power}), \emph{i.e.}, we use the covering argument, Chernoff bound,and union bound.

Define the distance between policies $d(\pi^1,\pi^2) \triangleq \max_{s\in \mathcal{S}, h \in [H]} \|\pi^1_h(\cdot|s) - \pi^2_h (\cdot|s)\|_{1}$. 
We construct the $\epsilon$-covering set $\tilde{\Pi}$ w.r.t. $ d$ such that
\begin{align}
    \label{inq: definition of tildePi}
    \forall \pi \in \Pi, \exists \tilde{\pi} \in \tilde{\Pi}, s.t. \quad d(\pi,\tilde{\pi}) \leq \epsilon.
\end{align}
Then we have
\begin{align}
    \label{inq: pi tildepi value gap}
    \forall i\in[n], \pi \in \Pi, \exists \tilde{\pi} \in \tilde{\Pi}, s.t. V_{i,1}^{\pi} (x_1) - V_{i,1}^{\tilde{\pi}} (x_1) \leq H \epsilon.
\end{align}
By the definition of the covering number, $\left|\tilde{\Pi}\right| =\mathcal{N}_\epsilon^\Pi$. By Chernoff bound and union bound over the policy set $\tilde{\Pi}$, we have with prob. at least $1-\frac{\delta}{3}$, for any $\tilde{\pi} \in \tilde{\Pi}$,
\begin{align}
    \label{inq: concentration in tildePi}
    \left|\frac{1}{n} \sum_{i=1}^{n} V^{\tilde{\pi}}_{i,1} (x_1) - \EE_{c \sim C} {V^{\tilde{\pi}}_{c,1} (x_1)} \right| \leq \sqrt{\frac{2\log(6\mathcal{N}_\epsilon^\Pi/\delta)}{n}}.
\end{align}

By Eq.(\ref{inq: pi tildepi value gap}) and Eq.(\ref{inq: concentration in tildePi}), $\forall i\in[n], \pi \in \Pi, \exists \tilde{\pi} \in \tilde{\Pi}$ with $\left|\tilde{\Pi}\right| =\mathcal{N}_\epsilon^\Pi,~ s.t. V_{i,1}^{\pi} (x_1) - V_{i,1}^{\tilde{\pi}} (x_1) \leq H \epsilon$, and with probability at least $1-\delta/3$, we have
\begin{align}
    &\left|\frac{1}{n} \sum_{i=1}^{n} V^{{\pi}}_{i,1} (x_1) - \EE_{c \sim C} {V^{{\pi}}_{c,1} (x_1) }\right| \notag \\
    &\leq  \left|\frac{1}{n} \sum_{i=1}^{n} V^{\tilde{\pi}}_{i,1} (s_1) - \EE_{c \sim C} {V^{{\tilde{\pi}}}_{c,1} (x_1) } \right| \notag\\
    & + \left|\frac{1}{n} \sum_{i=1}^{n} V^{{\pi}}_{i,1} (s_1) - \frac{1}{n}\sum_{i=1}^{n} V^{\tilde{\pi}}_{i,1} (s_1) \right|
    + \left|\EE_{c \sim C} {V^{{\tilde{\pi}}}_{c,1} (x_1) } - \EE_{c \sim C} {V^{{\pi}}_{c,1} (x_1) } \right| \notag\\
    &\leq  \sqrt{\frac{2\log(6\mathcal{N}_\epsilon^\Pi/\delta)}{n}}+ 2H\epsilon\,.\label{eq: generalization gap}
\end{align}
Therefore, we have for the first two terms, w.p. at least $1-\frac{2}{3}\delta$ we can upper bound them with $4H\epsilon+2\sqrt{\frac{2\log(6\mathcal{N}_\epsilon^\Pi/\delta)}{n}}$.

Then, what remains is to bound the term $\frac{1}{n}\sum_{i=1}^n \big(V_{i,1}^{\pi^*}(x_1)-V_{i,1}^{\pi^{\text{PERM}}}(x_1)\big)$.

First, by similar arguments, we have 
\begin{align}
    V_{i,1}^{\pi^*}(x_1)-V_{i,1}^{\pi^{\text{PERM}}}(x_1)&\leq \big(V_{i,1}^{{\pi}^*}(x_1)-V_{i,1}^{\tilde{\pi}^{\text{PERM}}}(x_1)\big)+|V_{i,1}^{\tilde{\pi}^{\text{PERM}}}(x_1)-V_{i,1}^{\pi^{\text{PERM}}}(x_1)|\notag\\
    &\leq H\epsilon+V_{i,1}^{{\pi}^*}(x_1)-V_{i,1}^{\tilde{\pi}^{\text{PERM}}}(x_1)\label{cover 1}\,,
\end{align}
where $ \tilde{\pi}^{\text{PERM}}\in\tilde{\Pi}$ such that $|V_{i,1}^{\tilde{\pi}^{\text{PERM}}}(x_1)-V_{i,1}^{\pi^{\text{PERM}}}(x_1)|\leq H\epsilon$.

By the definition of the oracle in Definition.\ref{def:oracle}, the algorithm design of Algo.\ref{alg:model based general} (e.g., we call oracle $\mathbb{O}(\cD_h, \hat V_{h+1}, \delta/(3nH\mathcal{N}_{(Hn)^{-1}}^\Pi))$), and use a union bound over $H$ steps, $n$ contexts, and $\mathcal{N}_{(Hn)^{-1}}^\Pi$ policies, we have: with probability at least $1-\delta/3$, the condition in Eq.(\ref{high prob 2}) holds (with the policy class $\Pi$ replaced by $\tilde{\Pi}$ (and $\epsilon=1/(Hn))$.

Then, we have
\begin{align}
    &\frac{1}{n}\sum_{i=1}^n\big(V_{i,1}^{\pi^*}(x_1)-V_{i,1}^{\tilde\pi^{\text{PERM}}}(x_1)\big)\notag \\
    &\leq \frac{1}{n}\sum_{i=1}^n\big(V_{i,1}^{\pi^*}(x_1)-\hat{V}_{i,1}^{\tilde\pi^{\text{PERM}}}(x_1)\big)\notag\\
    &=\frac{1}{n}\sum_{i=1}^n\big(V_{i,1}^{\pi^*}(x_1)-\hat{V}_{i,1}^{\pi^{\text{PERM}}}(x_1)\big)+\frac{1}{n}\sum_{i=1}^n\big(\hat{V}_{i,1}^{\pi^{\text{PERM}}}(x_1)-\hat{V}_{i,1}^{\tilde\pi^{\text{PERM}}}(x_1)\big)\notag\\
    &\leq\frac{1}{n}\sum_{i=1}^n\big(V_{i,1}^{\pi^*}(x_1)-\hat{V}_{i,1}^{\pi^{\text{PERM}}}(x_1)\big)+H\cdot\frac{1}{Hn}\notag\\
    &\leq \frac{1}{n}\sum_{i=1}^n\big(V_{i,1}^{\pi^*}(x_1)-\hat{V}_{i,1}^{\pi^*}(x_1)\big)+1/n\label{last term}\,,
\end{align}
where the first inequality holds because of the pessimism in Lemma \ref{pessimism lemma1}, the second inequality holds because $|\hat V_{i,1}^{\tilde{\pi}^{\text{PERM}}}(x_1)-\hat V_{i,1}^{\pi^{\text{PERM}}}(x_1)|\leq H\epsilon$ with $\epsilon$ here specified as $1/(Hn)$, and the last inequality holds because that in the algorithm design of Algo.\ref{alg:erm} we set $\pi^{\text{PERM}}=\argmax_{\pi\in\Pi}\frac{1}{n}\sum_{i=1}^n \hat{V}^\pi_{i,1}(x_1)$. 

Then what left is to bound $V_{i,1}^{\pi^*}(x_1)-\hat{V}_{i,1}^{\pi^*}(x_1)$. 

And using Lemma A.1 in \cite{jin2021pessimism}, we have
\begin{align}
    V_{i,1}^{\pi^*}(x_1)-\hat{V}_{i,1}^{\pi^*}(x_1)&=-\sum_{h=1}^H   \EE_{\hatpistar, \cM_i}\big[  \iota_{i,h}^{\pi^*}  (s_h,a_h) \biggiven s_1=x\big]+\sum_{h=1}^H   \EE_{\pi^*, \cM_i}\big[  \iota_{i,h}^{\pi^*}  (s_h,a_h) \biggiven s_1=x\big]\notag \\
    &\quad+ \sum_{h=1}^H \EE_{\pi^*, \cM_i}\big[ \langle \hat{Q}^{\pi^*}_{i,h}(s_h,\cdot) , \pi^*_h(\cdot\given s_h) - \pi^*_h(\cdot\given s_h) \rangle_{\cA} \biggiven s_1=x\big] \notag\\
    &\leq 2 \sum_{h=1}^H   \EE_{\pi^*, \cM_i}\big[  \Gamma_{i,h}   (s_h,a_h) \biggiven s_1=x\big]\label{final step}\,,
\end{align}
where in the last inequality we use Lemma \ref{lem:model_eval_err}.

Finally, with Eq.(\ref{decomposition}), Eq.(\ref{eq: generalization gap}), Eq.(\ref{cover 1}), Eq.(\ref{last term}), and Eq.(\ref{final step}), with $\epsilon$ set as $\frac{1}{nH}$, we can get w.p. at least $1-\delta$

\begin{align}\normalfont
&\EE_{c \sim C}{V_{c,1}^{\pi^*}(x_1)-V_{c,1}^{\pi^{\text{PERM}}}(x_1)}\notag \\
&\leq  \frac{5}{n} +2\sqrt{\frac{2\log(6\mathcal{N}_{(Hn)^{-1}}^\Pi/\delta)}{n}}+\frac{2}{n}\sum_{i=1}^n\sum_{h=1}^H\E{\pi^*,\cM_i}{\Gamma_{i,h}(s_h,a_h)|s_1=x_1}\notag\\
&\leq 7\sqrt{\frac{2\log(6\mathcal{N}_{(Hn)^{-1}}^\Pi/\delta)}{n}}+\frac{2}{n}\sum_{i=1}^n\sum_{h=1}^H\E{\pi^*,\cM_i}{\Gamma_{i,h}(s_h,a_h)|s_1=x_1}\,.
\notag
\end{align}
\end{proof}

\subsection{Proof of Theorem \ref{thm:model-free}}\label{proof:modelfree}
Our proof has two steps. First, we define that
\begin{align}
    \iota_{i,h}(x,a):=\mathbb{B}_{i,h} V_{i, h+1}(x,a) - Q_{i,h}(x,a)
\end{align}
Then we have the following lemma from \citet{jin2021pessimism}:
\begin{lemma}\label{lemma:pess}
Define the event $\cE$ as
\begin{align}
    \cE = \bigg\{\big|(\hat\BB \hat V^{\pi_i}_{i,h+1})(x,a) - (\BB_{i,h} \hat V^{\pi_i}_{i,h+1})(x,a)\big|\leq \Gamma_{i,h}(x,a)~ \forall (x,a)\in \cS\times \cA ,\forall h\in[H], \forall i\in[n]  \bigg\},\notag
\end{align}
    Then by selecting the input parameter $\xi = \delta/(Hn)$ in $\mathbb{O}$, we have $\PP(\cE)\geq 1-\delta$ and 
    \begin{align}
        0 \leq \iota_{i,h}(x,a) \leq 2\Gamma_{i,h}(x,a).\notag
    \end{align}
\end{lemma}
\begin{proof}
    The proof is the same as [Lemma 5.1, \citealt{jin2021pessimism}] with the probability assigned as $\delta/(Hn)$ and a union bound over $h\in[H], i\in[n]$.   
\end{proof}
Next lemma shows the difference between the value of the optimal policy $\pi^*$ and number $n$ of different policies $\pi_i$ for $n$ MDPs.

\begin{lemma}\label{lemma:3.1}
Let $\pi$ be an arbitrary policy. Then we have 
    \begin{align}
        \sum_{i=1}^n[V_{i,1}^{\pi}(x_1) - V_{i,1}^{\pi^i}(x_1)] &=\sum_{i=1}^n\sum_{h=1}^H \EE_{i,\pi}[\la Q_{i,h}(\cdot, \cdot), \pi_h(\cdot|\cdot) - \pi_{i,h}(\cdot| \cdot)\ra_{\cA}] \notag \\
        &\quad + \sum_{i=1}^n \sum_{h=1}^H (\EE_{i,\pi}[\iota_{i,h}(x_h, a_h)] - \EE_{i,\pi_i}[\iota_{i,h}(x_h, a_h)])
    \end{align}
\end{lemma}
\begin{proof}
    The proof is the same as Lemma 3.1 in \cite{jin2021pessimism} except substituting $\pi$ into the lemma. 
\end{proof}

We also have the following one-step lemma:
\begin{lemma}[Lemma 3.3, \citealt{cai2020provably}]
For any distribution $p^*, p \in \Delta(\cA)$, if $p'(\cdot)\propto p(\cdot)\cdot \exp(\alpha\cdot Q(x,\cdot))$, then
\begin{align}
    \la Q(x, \cdot), p^*(\cdot) - p(\cdot)\ra \leq \alpha H^2/2 + \alpha^{-1}\cdot \bigg(\text{KL}(p^*(\cdot)\|p(\cdot)) -\text{KL}(p^*(\cdot)\| p'(\cdot))  \bigg).\notag
\end{align}
\end{lemma}

Given the above lemmas, we begin our proof of Theorem \ref{thm:model-free}. 
\begin{proof}[Proof of Theorem \ref{thm:model-free}]
Combining Lemma \ref{lemma:pess} and Lemma \ref{lemma:3.1}, we have
\begin{align}
    &\sum_{i=1}^n[V_{i,1}^{\pi^*}(x_1) - V_{i,1}^{\pi^i}(x_1)] \notag \\
    &\leq \sum_{i=1}^n\sum_{h=1}^H \EE_{i,\pi^*}[\la Q_{i,h}, \pi_h^* - \pi_{i,h}\ra] + 2\sum_{i=1}^n \sum_{h=1}^H\EE_{i,\pi^*}[\Gamma_{i,h}(x_h,a_h)]\notag \\
    & \leq \sum_{i=1}^n\sum_{h=1}^H \alpha H^2/2 + \alpha^{-1}\EE_{i,\pi^*}[\text{KL}(\pi_h^*(\cdot|x_h)\|\pi_{i,h}(\cdot|x_h)) - \text{KL}(\pi_h^*(\cdot|x_h)\|\pi_{i+1,h}(\cdot|x_h))]\notag \\
    &\quad + 2\sum_{i=1}^n \sum_{h=1}^H\EE_{i,\pi^*}[\Gamma_{i,h}(x_h,a_h)]\notag \\
    & \leq \alpha H^3 n/2 + \alpha^{-1}\cdot \sum_{h=1}^H\EE_{i,\pi^*}[\text{KL}(\pi_h^*(\cdot|x_h)\|\pi_{1,h}(\cdot|x_h))] + 2\sum_{i=1}^n \sum_{h=1}^H\EE_{i,\pi^*}[\Gamma_{i,h}(x_h,a_h)]\notag \\
    & \leq \alpha H^3 n/2 + \alpha^{-1}H \log|A| + 2\sum_{i=1}^n \sum_{h=1}^H\EE_{i,\pi^*}[\Gamma_{i,h}(x_h,a_h)],\notag
\end{align}
where the last inequality holds since $\pi_{1,h}$ is the uniform distribution over $\cA$. Then, selecting $\alpha = 1/\sqrt{H^2n}$, we have
\begin{align}
    \sum_{i=1}^n[V_{i,1}^{\pi^*}(x_1) - V_{i,1}^{\pi^i}(x_1)] \leq 2\sqrt{n\log|A|H^2} + 2\sum_{i=1}^n \sum_{h=1}^H\EE_{i,\pi^*}[\Gamma_{i,h}(s_h,a_h)],\notag
\end{align}
which holds for the random selection of $\cD$ with probability at least $1-\delta$. Meanwhile, note that each MDP $M_i$ is drawn i.i.d. from $C$. Meanwhile, note that $\pi_i$ only depends on MDP $M_1, ..., M_{i-1}$. Therefore, by the standard online-to-batch conversion, we have
\begin{small}
    \begin{align}
    \PP\bigg(\frac{1}{n}\sum_{i=1}^n[V_{i,1}^{\pi^*}(x_1) - V_{i,1}^{\pi_i}(x_1)] + \bigg(\frac{1}{n} \sum_{i=1}^n\EE_{c \sim C} V_{c,1}^{\pi_i}(x_1)- \EE_{c \sim C} V_{c,1}^{\pi^*}(x_1)\bigg) \leq 2H\sqrt{\frac{2\log1/\delta}{n}}\bigg) \geq 1-\delta,\notag
\end{align}
\end{small}
which suggests that with probability at least $1-2\delta$, 
\begin{align}
    \EE_{c \sim C} V_{c,1}^{\pi^*}(x_1) - \frac{1}{n} \sum_{i=1}^n\EE_{c \sim C} V_{c,1}^{\pi_i}(x_1) \leq 2\sqrt{\frac{\log|A|H^2}{n}} + \frac{2}{n}\sum_{i=1}^n \sum_{h=1}^H\EE_{\pi^*}[\Gamma_{i,h}(x_h,a_h)] + 2\sqrt{\frac{2H\log1/\delta}{n}}.\notag
\end{align}
Therefore, by selecting $\pi^{\text{PPPO}}:=\text{random}(\pi_1, ..., \pi_n)$ and applying the Markov inequality, setting $\delta = 1/8$, we have our bound holds. 
\end{proof}

\section{Suboptimality bounds for real-world setups}\label{appendix:merge}
In this section we state and prove the suboptimality bounds we promised in Remarks \ref{rmk:merge} and \ref{rmk:mergefree}, where we merge the sampled contexts into $m$ groups (generally, $m<<n$) to reduce the computational complexity in practical settings. 

Assume $m|n$ and the $n$ contexts from offline dataset are equally partitioned into $m$ groups. We write the resulting average MDPs (see Proposition \ref{thm: nodistinguish}) for each group as $\bar\cM_1,\ldots,\bar\cM_m$. For each $\bar\cM_j$, we regard it as an individual context in the sense of (\ref{high prob 2}) and denote the resulting uncertainty quantifier and value function as ${\Gamma'}_{j,h}, {V'}^\pi_{j,h}$.

\begin{theorem}[Suboptimality bound for Remark \ref{rmk:merge}]\label{thm:permv}
Assume the same setting as Theorem \ref{thm:model based regret_upper_bound_general} with the original $n$ contexts grouped as $m$ contexts, and denote the resulting algorithm as PERM-$m$V. Then w.p. at least $1-\delta$, the output $\pi'$ of PERM-$m$V satisfies 
\begin{small}
    \begin{align}\normalfont
\text{SubOpt}(\pi')&\leq \underbrace{2\sqrt{\frac{2\log(6\mathcal{N}_{(Hm)^{-1}}^\Pi/\delta)}{n}}}_{I_1: \text{Supervised learning (SL) error}}+\underbrace{\frac{2}{m}\sum_{j=1}^m\sum_{h=1}^H\E{\pi^*,\bar\cM_j}{{\Gamma'}_{j,h}(s_h,a_h)|s_1=x_1}}_{I_2: \text{Reinforcement learning (RL) error}}\notag \\
&+ \underbrace{\frac{5}{m}+2 \sup_\pi \left|  \frac{1}{n}\sum_{i=1}^n V^{\pi}_{i,1}(x_1)-\frac{1}{m}\sum_{j=1}^m {V'}^{\pi}_{j,1}(x_1)\right|}_{\text{Additional approximation error}},\notag
\end{align}
\end{small}
where $\EE_{j,\pi^*}$ is w.r.t. the trajectory induced by $\pi^*$ with the transition $\bar\cP_j$ in the underlying average MDP $\bar\cM_j$.
\end{theorem}

\begin{proof}[Proof of Theorem \ref{thm:permv}]

Similar to the proof of Theorem \ref{thm:model based regret_upper_bound_general}, we decompose the suboptimality gap as follows
\begin{align}
    &\text{SubOpt}(\pi')\notag \\
    &=\EE_{c \sim C}{V_{c,1}^{\pi^*}(x_1)-V_{c,1}^{\pi'}(x_1)}\notag \\
    &=\EE_{c \sim C}{V_{c,1}^{\pi^*}(x_1)}-\frac{1}{n}\sum_{i=1}^n V^{\pi^*}_{i,1}(x_1) +\frac{1}{n}\sum_{i=1}^n V_{i,1}^{\pi'}(x_1)-\EE_{c \sim C}{V_{c,1}^{\pi'}(x_1)}\notag\\
    &\quad +\frac{1}{n}\sum_{i=1}^n V^{\pi^*}_{i,1}(x_1)-\frac{1}{m}\sum_{j=1}^m {V'}^{\pi^*}_{j,1}(x_1)+\frac{1}{m}\sum_{j=1}^m {V'}_{j,1}^{\pi'}(x_1)-\frac{1}{n}\sum_{i=1}^n V_{i,1}^{\pi'}(x_1)\notag\\
    &\quad +\frac{1}{m}\sum_{j=1}^m \big({V'}_{j,1}^{\pi^*}(x_1)-{V'}_{j,1}^{\pi'}(x_1)\big)\label{decomposition1}\,.
\end{align}

Note that we can bound the first and third lines of (\ref{decomposition1}) with the exactly same arguments as the proof of Theorem \ref{thm:model based regret_upper_bound_general}, the only notation-wise difference is that the uncertainty quantifier becomes $\Gamma'$ as we are operating on the level of average MDP $\bar\cM_j$.

The only thing left is to bound the second line of (\ref{decomposition1}). This is the same in spirit of the bound (\ref{eq:model based gap no context}), so that we can express the bound as follows
\begin{align}
    &\frac{1}{n}\sum_{i=1}^n V^{\pi^*}_{i,1}(x_1)-\frac{1}{m}\sum_{j=1}^m {V'}^{\pi^*}_{j,1}(x_1)+\frac{1}{m}\sum_{j=1}^m {V'}_{j,1}^{\pi'}(x_1)-\frac{1}{n}\sum_{i=1}^n V_{i,1}^{\pi'}(x_1)\notag\\
    &\leq 2 \sup_\pi \left|  \frac{1}{n}\sum_{i=1}^n V^{\pi}_{i,1}(x_1)-\frac{1}{m}\sum_{j=1}^m {V'}^{\pi}_{j,1}(x_1)\right|.\notag
\end{align}

To conclude, our final bound can be expressed as: with $\epsilon$ set as $\frac{1}{mH}$, we can get w.p. at least $1-\delta$
\begin{align}
    &\text{SubOpt}(\pi')\notag \\
    &\leq 2\sqrt{\frac{2\log(6\mathcal{N}_{(Hm)^{-1}}^\Pi/\delta)}{n}}+\frac{2}{m}\sum_{j=1}^m\sum_{h=1}^H\E{\pi^*,\bar\cM_j}{{\Gamma'}_{j,h}(s_h,a_h)|s_1=x_1}\notag\\
    &\quad +\frac{5}{m}+2 \sup_\pi \left|  \frac{1}{n}\sum_{i=1}^n V^{\pi}_{i,1}(x_1)-\frac{1}{m}\sum_{j=1}^m {V'}^{\pi}_{j,1}(x_1)\right|.\notag
\end{align}
\end{proof}

To prove the suboptimality bound for Remark \ref{rmk:mergefree}, we denote that the policies produced by PPPO after merging dataset to $m$ groups to be $\pi_1,\ldots,\pi_m$, and the original PPPO algorithm would produce the policies as $\pi'_1,\ldots,\pi'_n$. We assume that the merging of dataset from $n$ to $m$ groups is only to combine the consecutive $n/m$ terms from $\pi'_1,\ldots,\pi'_n$ and preserves the order.

\begin{theorem}[Suboptimality bound for Remark \ref{rmk:mergefree}]\label{thm:pppov}
    Assume the same setting as Theorem \ref{thm:model-free} with the original $n$ contexts grouped as $m$ contexts, and denote the resulting algorithm as PPPO-$m$V. Let ${\Gamma'}_{j,h}$ be the uncertainty quantifier returned by $\mathbb{O}$ through the PPPO-$m$V algorithm. Selecting $\alpha = 1/\sqrt{H^2m}$. Then selecting $\delta = 1/8$, w.p. at least $2/3$, we have
    \begin{small}
        \begin{align}
    \text{SubOpt}(\pi^{\text{PPPO}-mV}) &\leq 10\bigg(\underbrace{\sqrt{\frac{\log|\actions|H^2}{m}}}_{I_1: \text{SL error}} + \underbrace{\frac{1}{m}\sum_{j=1}^m \sum_{h=1}^H\E{j,\pi^*}{{\Gamma'}_{j,h}(s_h,a_h)|s_1=x_1}}_{I_2: \text{RL error}}\notag \\
    &+\sup_\pi\left| \frac{1}{n}\sum_{i=1}^nV_{i,1}^{\pi}(x_1)-\frac{1}{m}\sum_{j=1}^m{V'}_{j,1}^{\pi}(x_1)\right|+\frac{1}{n}\sum_{i=1}^n\sup_\pi\left| \mathbb{E}_c[V_{c,1}^{\pi}(x_1)]-V_{i,1}^{\pi}(x_1)\right|\notag \\
    &+\frac{1}{m}\sum_{j=1}^m\sup_\pi\left| \mathbb{E}_c[{V'}_{c,1}^{\pi}(x_1)]-{V'}_{j,1}^{\pi}(x_1)\right|\bigg).\notag
\end{align}
    \end{small}
where $\EE_{j,\pi^*}$ is w.r.t. the trajectory induced by $\pi^*$ with the transition $\bar\cP_j$ in the underlying MDP $\bar\cM_j$.
\end{theorem}

\begin{proof}[Proof of Theorem \ref{thm:pppov}]

Using the same arguments as in the proof of Theorem \ref{thm:model-free} with $\alpha=1/\sqrt{H^2m}$, we can derive the bound

$$\sum_{j=1}^m[{V'}_{j,1}^{\pi^*}(x_1) - {V'}_{j,1}^{\pi_j}(x_1)] \leq 2\sqrt{m\log|A|H^2} + 2\sum_{j=1}^m \sum_{h=1}^H\EE_{j,\pi^*}[{\Gamma'}_{j,h}(s_h,a_h)].$$

Leveraging this bound and online-to-batch, we obtain the following estimation

\begin{align}
    &\mathbb{E}_{c}[V^{\pi^\ast}_{c,1}(x_1)]-\frac{1}{m}\sum_{j=1}^{m}\mathbb{E}_{c}[V^{\pi_j}_{c,1}(x_1)]\notag \\
    =& \mathbb{E}_{c}[V^{\pi^\ast}_{c,1}(x_1)]-\frac{1}{n}\sum_{i=1}^{n}\mathbb{E}_{c}[V^{\pi'_i}_{c,1}(x_1)]+\frac{1}{n}\sum_{i=1}^{n}\mathbb{E}_{c}[V^{\pi'_i}_{c,1}(x_1)]-\frac{1}{m}\sum_{j=1}^{m}\mathbb{E}_{c}[V^{\pi_j}_{c,1}(x_1)]\notag\\
    \leq &2H\sqrt{\frac{2\log 1/\delta}{n}}+ \frac{1}{n}\sum_{i=1}^n\left( \mathbb{E}_c[V_{c,1}^{\pi'_i}(x_1)]-V_{i,1}^{\pi'_i}(x_1)\right)+ \frac{1}{n}\sum_{i=1}^nV_{i,1}^{\pi^*}(x_1)-\frac{1}{m}\sum_{j=1}^{m}\mathbb{E}_{c}[V^{\pi_j}_{c,1}(x_1)]\notag\\
    = & 2H\sqrt{\frac{2\log 1/\delta}{n}}+\frac{1}{n}\sum_{i=1}^nV_{i,1}^{\pi^*}(x_1)-\frac{1}{m}\sum_{j=1}^m{V'}_{j,1}^{\pi^*}(x_1)\notag \\
    &+\frac{1}{m}\sum_{j=1}^m{V'}_{j,1}^{\pi^*}(x_1)-\frac{1}{m}\sum_{j=1}^m{V'}_{j,1}^{\pi_j}(x_1)\notag \\
    &+ \frac{1}{n}\sum_{i=1}^n\left( \mathbb{E}_c[V_{c,1}^{\pi'_i}(x_1)]-V_{i,1}^{\pi'_i}(x_1)\right)+\frac{1}{m}\sum_{j=1}^m{V'}_{j,1}^{\pi_j}(x_1)-\frac{1}{m}\sum_{j=1}^{m}\mathbb{E}_{c}[V^{\pi_j}_{c,1}(x_1)]\notag \\
    \leq &2H\sqrt{\frac{2\log 1/\delta}{n}}+\sup_\pi\left| \frac{1}{n}\sum_{i=1}^nV_{i,1}^{\pi}(x_1)-\frac{1}{m}\sum_{j=1}^m{V'}_{j,1}^{\pi}(x_1)\right|\notag \\
    &+2\sqrt{\frac{\log|A|H^2}{m}} + \frac{2}{m}\sum_{j=1}^m \sum_{h=1}^H\EE_{j,\pi^*}[{\Gamma'}_{j,h}(s_h,a_h)]\notag \\
    &+ \frac{1}{n}\sum_{i=1}^n\sup_\pi\left| \mathbb{E}_c[V_{c,1}^{\pi}(x_1)]-V_{i,1}^{\pi}(x_1)\right|+ \frac{1}{m}\sum_{j=1}^m\sup_\pi\left| \mathbb{E}_c[{V'}_{c,1}^{\pi}(x_1)]-{V'}_{j,1}^{\pi}(x_1)\right|.\notag
\end{align}
Finally we apply Markov inequality and take $\delta=1/8$ as in the proof of Theorem \ref{thm:model-free}.
\end{proof}
\section{Results in Section \ref{sec:linear}}\label{proof:linear}
\subsection{Proof of Theorem \ref{thm:regret_upper_linear}}
    By \cite{jin2021pessimism}, the parameters specified as $\lambda=1,\quad \beta(\delta) = c\cdot dH\sqrt{\log(2dHK/\delta)}$, and applying union bound, we can get:
for Algo.\ref{alg:linear mdp model based}, with probability at least $1-\delta/3$
\begin{align}
    &\big|(\hat\BB_{i,h} \hat{V}^\pi_{i,h+1})(x,a) - (\BB_{i,h} \hat{V}^\pi_{i,h+1})(x,a)\big|\leq  \beta\big(\frac{\delta}{3nH\mathcal{N}_{(Hn)^{-1}}^\Pi}\big) \bigl(\phi(x,a)^\top \Lambda_{i,h}^{-1}\phi(x,a)\bigr)^{1/2}\,,\notag\\
    &\quad\text{for all}~i\in[n], \pi\in\tilde\Pi, (x,a)\in \cS\times \cA, h\in[H]\,,\label{high prob 3}
\end{align}
where $\tilde\Pi$ is the $\frac{1}{Hn}$-covering set of the policy space $\Pi$ w.r.t. distance $\mathrm d(\pi^1,\pi^2) = \max_{s\in \mathcal{S}, h \in [H]} \|\pi^1_h(\cdot|s) - \pi^2_h (\cdot|s)\|_{1}$.

Therefore, we can specify the $\Gamma_{i,h}(\cdot,\cdot)$ in Theorem \ref{thm:model based regret_upper_bound_general} with $\beta\big(\frac{\delta}{3nH\mathcal{N}_{(Hn)^{-1}}^\Pi}\big) \bigl(\phi(x,a)^\top \Lambda_{i,h}^{-1}\phi(x,a)\bigr)^{1/2}$, and follow the same process as the proof of Theorem \ref{thm:model based regret_upper_bound_general} to get the result for Algo.\ref{alg:erm} with subroutine Algo.\ref{alg:linear mdp model based}.

Similarly, we can get: we can get:
for Algo.\ref{alg:linear mdp model based}, with probability at least $1-1/4$
\begin{align}
    &\big|(\hat\BB_{i,h} \hat{V}_{i,h+1})(x,a) - (\BB_{i,h} \hat{V}_{i,h+1})(x,a)\big|\leq  \beta\big(\frac{\delta}{4nH}\big) \bigl(\phi(x,a)^\top \Lambda_{i,h}^{-1}\phi(x,a)\bigr)^{1/2}\,,\notag\\
    &\quad\text{for all}~i\in[n], (x,a)\in \cS\times \cA, h\in[H]\,.\label{high prob 3}
\end{align}
Therefore, we can specify the $\Gamma_{i,h}(\cdot,\cdot)$ in Theorem \ref{thm:model-free} with $\beta\big(\frac{\delta}{4nH}\big) \bigl(\phi(x,a)^\top \Lambda_{i,h}^{-1}\phi(x,a)\bigr)^{1/2}$ and follow the same process as the proof of Theorem \ref{thm:model-free} to get the result for Algo.\ref{alg:modelfree} with subroutine Algo.\ref{alg:linear mdp model based}.
\subsection{Proof of Corollary \ref{cor:well_explore}}
   By the assumption that $\cD_i$ is generated by behavior policy $\bar\pi_i$ which well-explores MDP $\cM_i$ with constant $c_i$ (where the well-explore is defined in Def.\ref{ass:wellexp}), the proof of Corollary 4.6 in \cite{jin2021pessimism}, and applying a union bound over $n$ contexts, we have that for Algo.\ref{alg:erm} with subroutine Algo.\ref{alg:linear mdp model based} w.p. at least $1-\delta/2$
\begin{align}
    &\|\phi(x,a)\|_{\Lambda_{i,h}^{-1}}\leq \sqrt{\frac{2d}{c_i K}}\notag\\&~\textrm{for all}~i\in[n], ~(x,a)\in \cS\times \cA \text{ and all } h\in[H]  \,,
\end{align}
and for Algo.\ref{alg:erm} with subroutine Algo.\ref{alg:linear mdp model based} w.p. at least $1-\delta/2$
\begin{align}
    &\|\phi(x,a)\|_{\Lambda_{i,h}^{-1}}\leq \sqrt{\frac{2dH}{c_i K}}\notag\\&~\textrm{for all}~i\in[n], ~(x,a)\in \cS\times \cA \text{ and all } h\in[H]  \,,
\end{align}
because we use the data splitting technique and we only utilize each trajectory once for one data tuple at some stage $h$, so we replace $K$ with $K/H$.

Then, the result follows by plugging the results above into Theorem \ref{thm:regret_upper_linear}.

\bibliographystyle{plainnat}
\bibliography{reference}

\begin{thebibliography}{73}
\providecommand{\natexlab}[1]{#1}
\providecommand{\url}[1]{\texttt{#1}}
\expandafter\ifx\csname urlstyle\endcsname\relax
  \providecommand{\doi}[1]{doi: #1}\else
  \providecommand{\doi}{doi: \begingroup \urlstyle{rm}\Url}\fi

\bibitem[Ajay et~al.(2021)Ajay, Yang, Nachum, and Agrawal]{ajay2021understanding}
Anurag Ajay, Ge~Yang, Ofir Nachum, and Pulkit Agrawal.
\newblock Understanding the generalization gap in visual reinforcement learning.
\newblock 2021.

\bibitem[Albrecht et~al.(2022)Albrecht, Fetterman, Fogelman, Kitanidis, Wr{\'o}blewski, Seo, Rosenthal, Knutins, Polizzi, Simon, et~al.]{albrecht2022avalon}
Joshua Albrecht, Abraham Fetterman, Bryden Fogelman, Ellie Kitanidis, Bartosz Wr{\'o}blewski, Nicole Seo, Michael Rosenthal, Maksis Knutins, Zack Polizzi, James Simon, et~al.
\newblock Avalon: A benchmark for rl generalization using procedurally generated worlds.
\newblock \emph{Advances in Neural Information Processing Systems}, 35:\penalty0 12813--12825, 2022.

\bibitem[Bai et~al.(2022)Bai, Wang, Yang, Deng, Garg, Liu, and Wang]{bai2022pessimistic}
Chenjia Bai, Lingxiao Wang, Zhuoran Yang, Zhihong Deng, Animesh Garg, Peng Liu, and Zhaoran Wang.
\newblock Pessimistic bootstrapping for uncertainty-driven offline reinforcement learning.
\newblock \emph{arXiv preprint arXiv:2202.11566}, 2022.

\bibitem[Beck et~al.(2023)Beck, Vuorio, Liu, Xiong, Zintgraf, Finn, and Whiteson]{beck2023survey}
Jacob Beck, Risto Vuorio, Evan~Zheran Liu, Zheng Xiong, Luisa Zintgraf, Chelsea Finn, and Shimon Whiteson.
\newblock A survey of meta-reinforcement learning.
\newblock \emph{arXiv preprint arXiv:2301.08028}, 2023.

\bibitem[Bengio et~al.(2020)Bengio, Pineau, and Precup]{Bengio2020InterferenceAG}
Emmanuel Bengio, Joelle Pineau, and Doina Precup.
\newblock Interference and generalization in temporal difference learning.
\newblock \emph{International Conference On Machine Learning}, 2020.

\bibitem[Bertran et~al.(2020)Bertran, Martinez, Phielipp, and Sapiro]{Bertrn2020InstanceBG}
Martin Bertran, Natalia Martinez, Mariano Phielipp, and Guillermo Sapiro.
\newblock Instance-based generalization in reinforcement learning.
\newblock \emph{Advances in Neural Information Processing Systems}, 33:\penalty0 11333--11344, 2020.

\bibitem[Bose et~al.(2024)Bose, Du, and Fazel]{bose2024offline}
Avinandan Bose, Simon~Shaolei Du, and Maryam Fazel.
\newblock Offline multi-task transfer rl with representational penalization.
\newblock \emph{arXiv preprint arXiv:2402.12570}, 2024.

\bibitem[Brunskill and Li(2013)]{brunskill2013sample}
Emma Brunskill and Lihong Li.
\newblock Sample complexity of multi-task reinforcement learning.
\newblock \emph{arXiv preprint arXiv:1309.6821}, 2013.

\bibitem[Cai et~al.(2020)Cai, Yang, Jin, and Wang]{cai2020provably}
Qi~Cai, Zhuoran Yang, Chi Jin, and Zhaoran Wang.
\newblock Provably efficient exploration in policy optimization.
\newblock In \emph{International Conference on Machine Learning}, pages 1283--1294. PMLR, 2020.

\bibitem[Cassandra et~al.(1994)Cassandra, Kaelbling, and Littman]{cassandra1994acting}
Anthony~R Cassandra, Leslie~Pack Kaelbling, and Michael~L Littman.
\newblock Acting optimally in partially observable stochastic domains.
\newblock In \emph{Aaai}, volume~94, pages 1023--1028, 1994.

\bibitem[Cobbe et~al.(2018)Cobbe, Klimov, Hesse, Kim, and Schulman]{cobbe2018quantifying}
Karl Cobbe, Oleg Klimov, Christopher Hesse, Taehoon Kim, and J.~Schulman.
\newblock Quantifying generalization in reinforcement learning.
\newblock \emph{International Conference On Machine Learning}, 2018.

\bibitem[Duan et~al.(2016)Duan, Schulman, Chen, Bartlett, Sutskever, and Abbeel]{duan2016rl}
Yan Duan, John Schulman, Xi~Chen, Peter~L Bartlett, Ilya Sutskever, and Pieter Abbeel.
\newblock Rl2: Fast reinforcement learning via slow reinforcement learning.
\newblock \emph{arXiv preprint arXiv:1611.02779}, 2016.

\bibitem[Duan et~al.(2020)Duan, Jia, and Wang]{duan2020minimax}
Yaqi Duan, Zeyu Jia, and Mengdi Wang.
\newblock Minimax-optimal off-policy evaluation with linear function approximation.
\newblock In \emph{International Conference on Machine Learning}, pages 2701--2709. PMLR, 2020.

\bibitem[Ehrenberg et~al.(2022)Ehrenberg, Kirk, Jiang, Grefenstette, and Rockt{\"a}schel]{ehrenberg2022study}
Andy Ehrenberg, Robert Kirk, Minqi Jiang, Edward Grefenstette, and Tim Rockt{\"a}schel.
\newblock A study of off-policy learning in environments with procedural content generation.
\newblock In \emph{ICLR Workshop on Agent Learning in Open-Endedness}, 2022.

\bibitem[Ernst et~al.(2005)Ernst, Geurts, and Wehenkel]{ernst2005tree}
Damien Ernst, Pierre Geurts, and Louis Wehenkel.
\newblock Tree-based batch mode reinforcement learning.
\newblock \emph{Journal of Machine Learning Research}, 6, 2005.

\bibitem[Finn et~al.(2017)Finn, Abbeel, and Levine]{finn2017model}
Chelsea Finn, Pieter Abbeel, and Sergey Levine.
\newblock Model-agnostic meta-learning for fast adaptation of deep networks.
\newblock In \emph{International conference on machine learning}, pages 1126--1135. PMLR, 2017.

\bibitem[Frans and Isola(2022)]{frans2022powderworld}
Kevin Frans and Phillip Isola.
\newblock Powderworld: A platform for understanding generalization via rich task distributions.
\newblock \emph{arXiv preprint arXiv:2211.13051}, 2022.

\bibitem[Ghasemipour et~al.(2022)Ghasemipour, Gu, and Nachum]{ghasemipour2022so}
Kamyar Ghasemipour, Shixiang~Shane Gu, and Ofir Nachum.
\newblock Why so pessimistic? estimating uncertainties for offline rl through ensembles, and why their independence matters.
\newblock \emph{Advances in Neural Information Processing Systems}, 35:\penalty0 18267--18281, 2022.

\bibitem[Ghosh et~al.(2021)Ghosh, Rahme, Kumar, Zhang, Adams, and Levine]{ghosh2021generalization}
Dibya Ghosh, Jad Rahme, Aviral Kumar, Amy Zhang, Ryan~P Adams, and Sergey Levine.
\newblock Why generalization in rl is difficult: Epistemic pomdps and implicit partial observability.
\newblock \emph{Advances in neural information processing systems}, 34:\penalty0 25502--25515, 2021.

\bibitem[Hu et~al.(2021)Hu, Chen, Jin, Li, and Wang]{hu2021near}
Jiachen Hu, Xiaoyu Chen, Chi Jin, Lihong Li, and Liwei Wang.
\newblock Near-optimal representation learning for linear bandits and linear rl.
\newblock In \emph{International Conference on Machine Learning}, pages 4349--4358. PMLR, 2021.

\bibitem[Ishfaq et~al.(2024)Ishfaq, Nguyen-Tang, Feng, Arora, Wang, Yin, and Precup]{ishfaq2024offline}
Haque Ishfaq, Thanh Nguyen-Tang, Songtao Feng, Raman Arora, Mengdi Wang, Ming Yin, and Doina Precup.
\newblock Offline multitask representation learning for reinforcement learning.
\newblock \emph{arXiv preprint arXiv:2403.11574}, 2024.

\bibitem[Jiang et~al.()Jiang, Kolter, and Raileanu]{jiang2022uncertainty}
Yiding Jiang, J~Zico Kolter, and Roberta Raileanu.
\newblock Uncertainty-driven exploration for generalization in reinforcement learning.
\newblock In \emph{Deep Reinforcement Learning Workshop NeurIPS 2022}.

\bibitem[Jin et~al.(2019)Jin, Yang, Wang, and Jordan]{jin2019provably}
Chi Jin, Zhuoran Yang, Zhaoran Wang, and Michael~I Jordan.
\newblock Provably efficient reinforcement learning with linear function approximation.
\newblock \emph{arXiv preprint arXiv:1907.05388}, 2019.

\bibitem[Jin et~al.(2021)Jin, Yang, and Wang]{jin2021pessimism}
Ying Jin, Zhuoran Yang, and Zhaoran Wang.
\newblock Is pessimism provably efficient for offline rl?
\newblock In \emph{International Conference on Machine Learning}, pages 5084--5096. PMLR, 2021.

\bibitem[Juliani et~al.(2019)Juliani, Khalifa, Berges, Harper, Teng, Henry, Crespi, Togelius, and Lange]{obstacletower}
Arthur Juliani, Ahmed Khalifa, Vincent{-}Pierre Berges, Jonathan Harper, Ervin Teng, Hunter Henry, Adam Crespi, Julian Togelius, and Danny Lange.
\newblock {Obstacle Tower: {A} Generalization Challenge in Vision, Control, and Planning}.
\newblock In \emph{IJCAI}, 2019.

\bibitem[Justesen et~al.(2018)Justesen, Torrado, Bontrager, Khalifa, Togelius, and Risi]{Justesen2018IlluminatingGI}
Niels Justesen, Ruben~Rodriguez Torrado, Philip Bontrager, Ahmed Khalifa, Julian Togelius, and Sebastian Risi.
\newblock Illuminating generalization in deep reinforcement learning through procedural level generation.
\newblock \emph{arXiv: Learning}, 2018.

\bibitem[Kirk et~al.(2023)Kirk, Zhang, Grefenstette, and Rockt{\"a}schel]{kirk2021generalisation}
Robert Kirk, Amy Zhang, Edward Grefenstette, and Tim Rockt{\"a}schel.
\newblock A survey of zero-shot generalisation in deep reinforcement learning.
\newblock \emph{Journal of Artificial Intelligence Research}, 76:\penalty0 201--264, 2023.

\bibitem[Kumar et~al.(2020)Kumar, Zhou, Tucker, and Levine]{kumar2020conservative}
Aviral Kumar, Aurick Zhou, George Tucker, and Sergey Levine.
\newblock Conservative q-learning for offline reinforcement learning.
\newblock \emph{Advances in Neural Information Processing Systems}, 33:\penalty0 1179--1191, 2020.

\bibitem[K{\"{u}}ttler et~al.(2020)K{\"{u}}ttler, Nardelli, Miller, Raileanu, Selvatici, Grefenstette, and Rockt{\"{a}}schel]{kuettler2020nethack}
Heinrich K{\"{u}}ttler, Nantas Nardelli, Alexander~H. Miller, Roberta Raileanu, Marco Selvatici, Edward Grefenstette, and Tim Rockt{\"{a}}schel.
\newblock {The NetHack Learning Environment}.
\newblock In \emph{Proceedings of the Conference on Neural Information Processing Systems (NeurIPS)}, 2020.

\bibitem[Lange et~al.(2012)Lange, Gabel, and Riedmiller]{lange2012batch}
Sascha Lange, Thomas Gabel, and Martin Riedmiller.
\newblock Batch reinforcement learning.
\newblock In \emph{Reinforcement learning: State-of-the-art}, pages 45--73. Springer, 2012.

\bibitem[Lee et~al.(2020)Lee, Lee, Shin, and Lee]{lee2020network}
Kimin Lee, Kibok Lee, Jinwoo Shin, and Honglak Lee.
\newblock Network randomization: A simple technique for generalization in deep reinforcement learning.
\newblock In \emph{International Conference on Learning Representations. https://openreview. net/forum}, 2020.

\bibitem[Levine et~al.(2020)Levine, Kumar, Tucker, and Fu]{levine2020offline}
Sergey Levine, Aviral Kumar, George Tucker, and Justin Fu.
\newblock Offline reinforcement learning: Tutorial, review, and perspectives on open problems.
\newblock \emph{arXiv preprint arXiv:2005.01643}, 2020.

\bibitem[Liu et~al.(2020)Liu, Swaminathan, Agarwal, and Brunskill]{liu2020provably}
Yao Liu, Adith Swaminathan, Alekh Agarwal, and Emma Brunskill.
\newblock Provably good batch off-policy reinforcement learning without great exploration.
\newblock \emph{Advances in neural information processing systems}, 33:\penalty0 1264--1274, 2020.

\bibitem[Lu et~al.(2021)Lu, Huang, and Du]{lu2021power}
Rui Lu, Gao Huang, and Simon~S Du.
\newblock On the power of multitask representation learning in linear mdp.
\newblock \emph{arXiv preprint arXiv:2106.08053}, 2021.

\bibitem[Lu et~al.(2025)Lu, Yue, Zhao, Du, and Huang]{lu2025towards}
Rui Lu, Yang Yue, Andrew Zhao, Simon Du, and Gao Huang.
\newblock Towards understanding the benefit of multitask representation learning in decision process.
\newblock \emph{arXiv preprint arXiv:2503.00345}, 2025.

\bibitem[Lyle et~al.(2022)Lyle, Rowland, Dabney, Kwiatkowska, and Gal]{lyle2022learning}
Clare Lyle, Mark Rowland, Will Dabney, Marta Kwiatkowska, and Yarin Gal.
\newblock Learning dynamics and generalization in deep reinforcement learning.
\newblock In \emph{International Conference on Machine Learning}, pages 14560--14581. PMLR, 2022.

\bibitem[Machado et~al.(2018)Machado, Bellemare, Talvitie, Veness, Hausknecht, and Bowling]{Machado2018RevisitingTA}
Marlos~C. Machado, Marc~G. Bellemare, Erik Talvitie, Joel Veness, Matthew~J. Hausknecht, and Michael~H. Bowling.
\newblock Revisiting the arcade learning environment: Evaluation protocols and open problems for general agents.
\newblock In \emph{IJCAI}, 2018.

\bibitem[Malik et~al.(2021)Malik, Li, and Ravikumar]{malik2021generalizable}
Dhruv Malik, Yuanzhi Li, and Pradeep Ravikumar.
\newblock When is generalizable reinforcement learning tractable?
\newblock \emph{Advances in Neural Information Processing Systems}, 34, 2021.

\bibitem[Mazoure et~al.(2022)Mazoure, Kostrikov, Nachum, and Tompson]{mazoure2022improving}
Bogdan Mazoure, Ilya Kostrikov, Ofir Nachum, and Jonathan~J Tompson.
\newblock Improving zero-shot generalization in offline reinforcement learning using generalized similarity functions.
\newblock \emph{Advances in Neural Information Processing Systems}, 35:\penalty0 25088--25101, 2022.

\bibitem[Mediratta et~al.(2023)Mediratta, You, Jiang, and Raileanu]{mediratta2023generalization}
Ishita Mediratta, Qingfei You, Minqi Jiang, and Roberta Raileanu.
\newblock The generalization gap in offline reinforcement learning.
\newblock \emph{arXiv preprint arXiv:2312.05742}, 2023.

\bibitem[Nguyen-Tang and Arora(2024)]{nguyen2024sample}
Thanh Nguyen-Tang and Raman Arora.
\newblock On sample-efficient offline reinforcement learning: Data diversity, posterior sampling and beyond.
\newblock \emph{Advances in Neural Information Processing Systems}, 36, 2024.

\bibitem[Nichol et~al.(2018)Nichol, Pfau, Hesse, Klimov, and Schulman]{Nichol2018GottaLF}
Alex Nichol, V.~Pfau, Christopher Hesse, O.~Klimov, and John Schulman.
\newblock Gotta learn fast: A new benchmark for generalization in rl.
\newblock \emph{ArXiv}, abs/1804.03720, 2018.

\bibitem[Osband et~al.(2016)Osband, Blundell, Pritzel, and Van~Roy]{osband2016deep}
Ian Osband, Charles Blundell, Alexander Pritzel, and Benjamin Van~Roy.
\newblock Deep exploration via bootstrapped dqn.
\newblock \emph{Advances in neural information processing systems}, 29, 2016.

\bibitem[Packer et~al.(2019)Packer, Gao, Kos, Krähenbühl, Koltun, and Song]{Packer2018AssessingGI}
Charles Packer, Katelyn Gao, Jernej Kos, Philipp Krähenbühl, Vladlen Koltun, and Dawn Song.
\newblock Assessing generalization in deep reinforcement learning.
\newblock \emph{ICLR}, 2019.

\bibitem[Pinto et~al.(2017)Pinto, Davidson, Sukthankar, and Gupta]{pinto2017robust}
Lerrel Pinto, James Davidson, Rahul Sukthankar, and Abhinav Gupta.
\newblock Robust adversarial reinforcement learning.
\newblock In \emph{International conference on machine learning}, pages 2817--2826. PMLR, 2017.

\bibitem[Rajeswaran et~al.(2017)Rajeswaran, Lowrey, Todorov, and Kakade]{Rajeswaran2017TowardsGA}
Aravind Rajeswaran, Kendall Lowrey, Emanuel Todorov, and Sham~M. Kakade.
\newblock Towards generalization and simplicity in continuous control.
\newblock In \emph{Advances in Neural Information Processing Systems 30: Annual Conference on Neural Information Processing Systems 2017, December 4-9, 2017, Long Beach, CA, {USA}}, pages 6550--6561, 2017.

\bibitem[Rashidinejad et~al.(2021)Rashidinejad, Zhu, Ma, Jiao, and Russell]{rashidinejad2021bridging}
Paria Rashidinejad, Banghua Zhu, Cong Ma, Jiantao Jiao, and Stuart Russell.
\newblock Bridging offline reinforcement learning and imitation learning: A tale of pessimism.
\newblock \emph{Advances in Neural Information Processing Systems}, 34:\penalty0 11702--11716, 2021.

\bibitem[Rezaeifar et~al.(2022)Rezaeifar, Dadashi, Vieillard, Hussenot, Bachem, Pietquin, and Geist]{rezaeifar2022offline}
Shideh Rezaeifar, Robert Dadashi, Nino Vieillard, L{\'e}onard Hussenot, Olivier Bachem, Olivier Pietquin, and Matthieu Geist.
\newblock Offline reinforcement learning as anti-exploration.
\newblock In \emph{Proceedings of the AAAI Conference on Artificial Intelligence}, volume~36, pages 8106--8114, 2022.

\bibitem[Riedmiller(2005)]{riedmiller2005neural}
Martin Riedmiller.
\newblock Neural fitted q iteration--first experiences with a data efficient neural reinforcement learning method.
\newblock In \emph{Machine Learning: ECML 2005: 16th European Conference on Machine Learning, Porto, Portugal, October 3-7, 2005. Proceedings 16}, pages 317--328. Springer, 2005.

\bibitem[Samvelyan et~al.(2021)Samvelyan, Kirk, Kurin, Parker-Holder, Jiang, Hambro, Petroni, K{\"u}ttler, Grefenstette, and Rockt{\"a}schel]{samvelyan2021minihack}
Mikayel Samvelyan, Robert Kirk, Vitaly Kurin, Jack Parker-Holder, Minqi Jiang, Eric Hambro, Fabio Petroni, Heinrich K{\"u}ttler, Edward Grefenstette, and Tim Rockt{\"a}schel.
\newblock Minihack the planet: A sandbox for open-ended reinforcement learning research.
\newblock \emph{arXiv preprint arXiv:2109.13202}, 2021.

\bibitem[Schulman et~al.(2017)Schulman, Wolski, Dhariwal, Radford, and Klimov]{schulman2017proximal}
John Schulman, Filip Wolski, Prafulla Dhariwal, Alec Radford, and Oleg Klimov.
\newblock Proximal policy optimization algorithms.
\newblock \emph{arXiv preprint arXiv:1707.06347}, 2017.

\bibitem[Song et~al.(2020)Song, Jiang, Tu, Du, and Neyshabur]{Song2020ObservationalOI}
Xingyou Song, Yiding Jiang, Stephen Tu, Yilun Du, and Behnam Neyshabur.
\newblock Observational overfitting in reinforcement learning.
\newblock In \emph{International Conference on Learning Representations}, 2020.
\newblock URL \url{https://openreview.net/forum?id=HJli2hNKDH}.

\bibitem[Tirinzoni et~al.(2020)Tirinzoni, Poiani, and Restelli]{tirinzoni2020sequential}
Andrea Tirinzoni, Riccardo Poiani, and Marcello Restelli.
\newblock Sequential transfer in reinforcement learning with a generative model.
\newblock In \emph{International Conference on Machine Learning}, pages 9481--9492. PMLR, 2020.

\bibitem[Touati et~al.(2023)Touati, Rapin, and Ollivier]{touati2023does}
Ahmed Touati, J{\'e}r{\'e}my Rapin, and Yann Ollivier.
\newblock Does zero-shot reinforcement learning exist?
\newblock In \emph{ICLR}, 2023.

\bibitem[Uehara and Sun(2021)]{uehara2021pessimistic}
Masatoshi Uehara and Wen Sun.
\newblock Pessimistic model-based offline reinforcement learning under partial coverage.
\newblock \emph{arXiv preprint arXiv:2107.06226}, 2021.

\bibitem[Uehara et~al.(2021)Uehara, Zhang, and Sun]{uehara2021representation}
Masatoshi Uehara, Xuezhou Zhang, and Wen Sun.
\newblock Representation learning for online and offline rl in low-rank mdps.
\newblock \emph{arXiv preprint arXiv:2110.04652}, 2021.

\bibitem[Wang et~al.(2019)Wang, Zheng, Xiong, and Socher]{wang2019generalization}
Huan Wang, Stephan Zheng, Caiming Xiong, and Richard Socher.
\newblock On the generalization gap in reparameterizable reinforcement learning.
\newblock In \emph{International Conference on Machine Learning}, pages 6648--6658. PMLR, 2019.

\bibitem[Wu et~al.(2021)Wu, Zhai, Srivastava, Susskind, Zhang, Salakhutdinov, and Goh]{wu2021uncertainty}
Yue Wu, Shuangfei Zhai, Nitish Srivastava, Joshua Susskind, Jian Zhang, Ruslan Salakhutdinov, and Hanlin Goh.
\newblock Uncertainty weighted actor-critic for offline reinforcement learning.
\newblock \emph{arXiv preprint arXiv:2105.08140}, 2021.

\bibitem[Xie et~al.(2021{\natexlab{a}})Xie, Cheng, Jiang, Mineiro, and Agarwal]{xie2021bellman}
Tengyang Xie, Ching-An Cheng, Nan Jiang, Paul Mineiro, and Alekh Agarwal.
\newblock Bellman-consistent pessimism for offline reinforcement learning.
\newblock \emph{Advances in neural information processing systems}, 34:\penalty0 6683--6694, 2021{\natexlab{a}}.

\bibitem[Xie et~al.(2021{\natexlab{b}})Xie, Jiang, Wang, Xiong, and Bai]{xie2021policy}
Tengyang Xie, Nan Jiang, Huan Wang, Caiming Xiong, and Yu~Bai.
\newblock Policy finetuning: Bridging sample-efficient offline and online reinforcement learning.
\newblock \emph{Advances in neural information processing systems}, 34:\penalty0 27395--27407, 2021{\natexlab{b}}.

\bibitem[Yan et~al.(2023)Yan, Li, Chen, and Fan]{yan2023efficacy}
Yuling Yan, Gen Li, Yuxin Chen, and Jianqing Fan.
\newblock The efficacy of pessimism in asynchronous q-learning.
\newblock \emph{IEEE Transactions on Information Theory}, 2023.

\bibitem[Yang and Wang(2019)]{yang2019sample}
Lin Yang and Mengdi Wang.
\newblock Sample-optimal parametric q-learning using linearly additive features.
\newblock In \emph{International Conference on Machine Learning}, pages 6995--7004, 2019.

\bibitem[Yang et~al.(2023)Yang, Yong, Ma, Hu, Zhang, and Zhang]{yang2023essential}
Rui Yang, Lin Yong, Xiaoteng Ma, Hao Hu, Chongjie Zhang, and Tong Zhang.
\newblock What is essential for unseen goal generalization of offline goal-conditioned rl?
\newblock In \emph{International Conference on Machine Learning}, pages 39543--39571. PMLR, 2023.

\bibitem[Yarats et~al.(2022)Yarats, Brandfonbrener, Liu, Laskin, Abbeel, Lazaric, and Pinto]{yarats2022don}
Denis Yarats, David Brandfonbrener, Hao Liu, Michael Laskin, Pieter Abbeel, Alessandro Lazaric, and Lerrel Pinto.
\newblock Don't change the algorithm, change the data: Exploratory data for offline reinforcement learning.
\newblock \emph{arXiv preprint arXiv:2201.13425}, 2022.

\bibitem[Ye et~al.(2020)Ye, Khalifa, Bontrager, and Togelius]{ye2020rotation}
Chang Ye, Ahmed Khalifa, Philip Bontrager, and Julian Togelius.
\newblock Rotation, translation, and cropping for zero-shot generalization.
\newblock In \emph{2020 IEEE Conference on Games (CoG)}, pages 57--64. IEEE, 2020.

\bibitem[Ye et~al.(2023)Ye, Chen, Wang, and Du]{ye2023power}
Haotian Ye, Xiaoyu Chen, Liwei Wang, and Simon~Shaolei Du.
\newblock On the power of pre-training for generalization in rl: provable benefits and hardness.
\newblock In \emph{International Conference on Machine Learning}, pages 39770--39800. PMLR, 2023.

\bibitem[Yin et~al.(2022)Yin, Duan, Wang, and Wang]{yin2022near}
Ming Yin, Yaqi Duan, Mengdi Wang, and Yu-Xiang Wang.
\newblock Near-optimal offline reinforcement learning with linear representation: Leveraging variance information with pessimism.
\newblock \emph{arXiv preprint arXiv:2203.05804}, 2022.

\bibitem[Yu et~al.(2020)Yu, Thomas, Yu, Ermon, Zou, Levine, Finn, and Ma]{yu2020mopo}
Tianhe Yu, Garrett Thomas, Lantao Yu, Stefano Ermon, James~Y Zou, Sergey Levine, Chelsea Finn, and Tengyu Ma.
\newblock Mopo: Model-based offline policy optimization.
\newblock \emph{Advances in Neural Information Processing Systems}, 33:\penalty0 14129--14142, 2020.

\bibitem[Zanette et~al.(2021)Zanette, Wainwright, and Brunskill]{zanette2021provable}
Andrea Zanette, Martin~J Wainwright, and Emma Brunskill.
\newblock Provable benefits of actor-critic methods for offline reinforcement learning.
\newblock \emph{Advances in neural information processing systems}, 34:\penalty0 13626--13640, 2021.

\bibitem[Zhang et~al.(2018{\natexlab{a}})Zhang, Ballas, and Pineau]{Zhang2018ADO}
Amy Zhang, Nicolas Ballas, and Joelle Pineau.
\newblock A dissection of overfitting and generalization in continuous reinforcement learning.
\newblock \emph{ArXiv}, abs/1806.07937, 2018{\natexlab{a}}.

\bibitem[Zhang and Wang(2021)]{zhang2021provably}
Chicheng Zhang and Zhi Wang.
\newblock Provably efficient multi-task reinforcement learning with model transfer.
\newblock \emph{Advances in Neural Information Processing Systems}, 34, 2021.

\bibitem[Zhang et~al.(2018{\natexlab{b}})Zhang, Vinyals, Munos, and Bengio]{Zhang2018ASO}
Chiyuan Zhang, Oriol Vinyals, R{\'e}mi Munos, and Samy Bengio.
\newblock A study on overfitting in deep reinforcement learning.
\newblock \emph{ArXiv}, abs/1804.06893, 2018{\natexlab{b}}.

\bibitem[Zhang et~al.(2023)Zhang, He, Zhou, Zhang, and Gu]{zhang2023provably}
Weitong Zhang, Jiafan He, Dongruo Zhou, Amy Zhang, and Quanquan Gu.
\newblock Provably efficient representation selection in low-rank markov decision processes: from online to offline rl.
\newblock In \emph{Proceedings of the Thirty-Ninth Conference on Uncertainty in Artificial Intelligence}, pages 2488--2497, 2023.

\end{thebibliography}

\end{document}